\documentclass[twoside,11pt]{article}

\usepackage{amsmath}
\usepackage{mathtools}
\usepackage{amssymb}
\usepackage{bbm}
\usepackage{subcaption}
\usepackage{float}
\usepackage{listings}
\usepackage[ruled]{algorithm2e}
\usepackage{tabularx}

\usepackage{jmlr2e}

\usepackage[x11names]{xcolor}
\newcommand{\ELnote}[1]{\textcolor{blue}{[{\em {\bf **EL Note:} #1}]}}

\newtheorem{assumption}{Assumption}

\DeclareMathOperator*{\argmin}{arg\,min}    
\DeclareMathOperator*{\argmax}{arg\,max}    

\newcommand{\eps}{\varepsilon}
\newcommand{\1}{\mathbbm{1}}

\newcommand{\E}{\mathbb{E}}

\newcommand{\N}{\mathbb{N}}

\renewcommand{\P}{\mathbb{P}}

\newcommand{\R}{\mathbb{R}}

\newcommand{\Z}{\mathbb{Z}}

\newcommand{\Ac}{\mathcal{A}}

\newcommand{\Fc}{\mathcal{F}}

\newcommand{\Nc}{\mathcal{N}}

\newcommand{\Qc}{\mathcal{Q}}

\usepackage{lastpage}
\jmlrheading{22}{2021}{1-\pageref{LastPage}}{10/20; Revised
6/21}{7/21}{20-1233}{Eric Lybrand and Rayan Saab}
\ShortHeadings{A Greedy Algorithm for Quantizing Neural Networks}{Lybrand and Saab}

\firstpageno{1}

\begin{document}
\lstset{language=Python}
\title{A Greedy Algorithm for Quantizing Neural Networks}

\author{\name Eric Lybrand \email elybrand@ucsd.edu \\
       \addr Department of Mathematics\\
       University of California, San Diego\\
       San Diego, CA 92121, USA
       \AND
       \name Rayan Saab \email rsaab@ucsd.edu \\
       \addr Department of Mathematics, and\\
       Halicioglu Data Science Institute\\
       University of California, San Diego\\
       San Diego, CA 92121, USA}

\editor{Gal Elidan}

\maketitle

\begin{abstract}
    We propose a new computationally efficient method for quantizing the weights of pre-trained neural networks that is general enough to handle  both multi-layer perceptrons and convolutional neural networks. Our method deterministically quantizes layers in an iterative fashion with no complicated re-training required. Specifically, we  quantize each neuron, or hidden unit, using a greedy path-following algorithm. This simple algorithm is equivalent to running a dynamical system, which we prove is stable for quantizing a single-layer neural network (or, alternatively, for quantizing the first layer of a multi-layer network) when the training data are Gaussian. We show that under these assumptions, the  quantization error decays with the width of the layer, i.e., its level of over-parametrization. We
    provide numerical experiments, on multi-layer networks, to illustrate the performance of our methods on MNIST and CIFAR10 data, as well as for quantizing the VGG16 network using ImageNet data.
\end{abstract}

\begin{keywords}
  quantization, neural networks, deep learning, stochastic control, discrepancy theory
\end{keywords}

\section{Introduction}

 Deep neural networks have taken the world by storm. They  outperform competing algorithms on applications ranging from speech recognition and translation to autonomous vehicles and even games, where they have beaten the best human players at, e.g., Go (see,  \citealt{lecun2015deep, goodfellow2016deep, schmidhuber2015deep, silver2016mastering}).  Such spectacular performance comes at a cost. Deep neural networks require a lot of computational power to train, memory to store, and power to run (e.g., \citealt{han2015deep, kim2015compression, gupta2015deep, courbariaux2015binaryconnect}). They are  painstakingly trained on powerful computing devices and then either run on these powerful devices or on the cloud. Indeed, it is well-known that the expressivity of a network depends on its architecture \citep{baldi2019capacity}. Larger networks can capture more complex behavior \citep{cybenko1989approximation} and therefore, for example, they generally learn better classifiers. The trade off, of course, is that larger networks require more memory for storage as well as more power to run computations. Those who design neural networks for the purpose of loading them onto a particular device must therefore account for the device's memory capacity, processing power, and power consumption. A deep neural network might yield a more accurate classifier, but it may require too much power to be run often without draining a device's battery. On the other hand, there is much to be gained in building  networks directly into hardware, for example as speech recognition or translation chips on mobile or handheld devices or hearing aids. Such mobile applications also impose restrictions on the amount of memory a neural network can use as well as its power consumption. 
\bigskip

\if{    Advances in computing and processing speed have brought about an era where complex classifiers such as neural networks can be trained on incredibly large data sets. The abundance of training data has enabled machine learning practitioners to train state-of-the-art image classifiers that only a decade ago would have seemed too challenging. Already neural networks are being embedded into various technologies such as self-driving cars and mobile phones \ELnote{citations!}. It is well-known that the expressivity of a network depends on its architecture. Larger networks can capture more complex behavior and therefore, in general, learn better classifiers \ELnote{citation}. The trade off, of course, is that larger networks require more memory for storage as well as more power to run computations to generate classifications. Those who design neural networks for the purpose of loading them onto a particular device must therefore account for the device's memory capacity, processing power, and power consumption. A deep neural network might yield a more accurate classifier, but it may require too much power to be run often without draining a device's battery, for example.}\fi
        
    This tension between network expressivity and the cost of computation has naturally posed the question of whether neural networks can be compressed without compromising their performance. Given that neural networks require computing many matrix-vector multiplications, arguably one of the most impactful changes would be to quantize the weights in the neural network. In the extreme case, replacing each \(32\)-bit floating point weight with a single bit would reduce the memory  required for storing a network by a factor of \(32\) and simplify scalar multiplications in the matrix-vector product. It is not clear at first glance, however, that there even exists a procedure for quantizing the weights that does not dramatically affect the network's performance.
    
    \subsection{Contributions}
    The goal of this paper is to propose a framework for quantizing neural networks without sacrificing their predictive power, and to provide theoretical justification for our framework. Specifically, 
    \begin{itemize}
        \item We propose a novel algorithm in \eqref{eq: first layer dynamical system} and \eqref{eq: hidden layer dynamical system} for sequentially quantizing layers of a pre-trained neural network in a data-dependent manner. This algorithm requires no retraining of the network, requires tuning only \(2\) hyperparameters---namely, the number of bits used to represent a weight and the radius of the quantization alphabet---and has a run time complexity of \(O(Nm)\) operations per layer. Here, \(N\) is the ambient dimension of the inputs, or equivalently, the number of features per input sample of the layer, while \(m\) is the number of training samples used to learn the quantization. This \(O(Nm)\) bound is optimal in the sense that any data-dependent quantization algorithm requires reading the \(Nm\) entries of the training data matrix. Furthermore, this algorithm is parallelizable across neurons in a given layer.
        \item We establish upper bounds on the relative training error in Theorem \ref{thm: first layer relative error} and the generalization error in Theorem \ref{thm: generalize} when quantizing the first layer of a neural network that hold with high probability when the training data are Gaussian. Additionally, these bounds make explicit how the relative training error and generalization error decay as a function of the overparametrization of the data.
        \item We provide numerical simulations in Section \ref{sec: numerics} for quantizing networks trained on the benchmark data sets MNIST and CIFAR10 using both multilayer perceptrons and convolutional neural networks. We quantize all layers of the neural networks in these numerical simulations to demonstrate that the quantized networks generalize very well even when the data are not Gaussian.
    \end{itemize}
    
    \section{Notation}
    Throughout the paper, we will use the following notation. Universal constants will be denoted as \(C, c\) and their values may change from line to line. For real valued quantities \(x, y\), we write \(x \lesssim y\) when we mean that \(x \leq C y\) and \(x \propto y\) when we mean \(c y \leq x \leq C y\). For any natural number \(m \in \N\), we denote the set \(\{1, \hdots, m\}\) by \([m]\). For column vectors \(u, v \in \R^{m}\), the Euclidean inner product is denoted by \(\langle u, v \rangle = u^T v = \sum_{j=1}^{m} u_j v_j\), the \(\ell_2\)-norm by \(\|u\|_2 = \sqrt{\sum_{j=1}^{m} u_j^2}\), the \(\ell_1\)-norm by \(\|u\|_1 = \sum_{j=1}^{m} |u_j|\), and the \(\ell_{\infty}\)-norm by \(\|u\|_{\infty} = \max_{j\in [m]} |u_j|\). \(B(x, r)\) will denote the \(\ell_2\)-ball centered at \(x\) with radius \(r\) and we will use the notation \(B_2^{m} := B(0,1) \subset \R^{m}\).  For a sequence of vectors \(u_t \in \R^{m}\) with \(t \in \Z\), the backwards difference operator \(\Delta\) acts by \(\Delta u_t = u_{t} - u_{t-1}\).  For a matrix \(X \in \R^{m \times N}\) we will denote the rows using lowercase characters \(x_t\) and the columns with uppercase characters \(X_t\). For two matrices \(X, Y \in \R^{m \times N}\) we denote the Frobenius norm by \(\|X-Y\|_F := \sqrt{\sum_{i,j}|X_{i,j}-Y_{i,j}|^2}\).
    \(\Phi\) will denote a $L$-layer neural network, or multilayer perceptron, which acts on data \(x \in \R^{N_0}\) via 
    \[\Phi(x) := \varphi \circ A^{(L)} \circ \cdots \circ \varphi \circ A^{(1)}(x).\]
    Here, \(\varphi: \R \to \R\) is a rectifier which acts on each component of a vector, \(A^{(\ell)}\) is an affine operator with  \(A^{(\ell)}(v) = v^T W^{(\ell)}+b^{(\ell)T}\) and \(W^{(\ell)} \in \R^{N_{\ell} \times N_{\ell+1}}\) is the \(\ell^{th}\) layer's weight matrix, \(b^{(\ell)} \in \R^{N_{\ell+1}}\) is the bias.

\section{Background}

    While there are a handful of empirical studies on quantizing neural networks, the mathematical literature on the subject is still in its infancy. In practice there appear to be three different paradigms for quantizing neural networks. These include quantizing the gradients during training, quantizing the activation functions, and quantizing the weights either during or after training. \cite{guo2018survey} presents an overview of these different paradigms. Any quantization that occurs during training introduces issues regarding the convergence of the learning algorithm. In the case of using quantized gradients, it is important to choose an appropriate codebook for the gradient prior to training to ensure stochastic gradient descent converges to a local minimum. When using quantized activation functions, one must suitably modify backpropagation since the activation functions are no longer differentiable. Further, enforcing the weights to be discrete during training also causes problems for backpropagation which assumes no such restriction. In any of these cases, it will be necessary to carefully choose hyperparameters and modify the training algorithm beyond what is necessary to train unquantized neural networks. In contrast to these approaches, our result allows the practitioner to train neural networks in any fashion they choose and quantizes the trained network afterwards. Our quantization algorithm only requires tuning the number of bits that are used to represent a weight and the radius of the quantization alphabet. We now turn to surveying approaches similar to ours which quantize weights after training.
    
    A natural question to ask is whether or not for every neural network there exists a quantized representation that approximates it well on a given data set. It turns out that a partial answer to this question lies in the field of discrepancy theory. Ignoring bias terms for now, let's look at quantizing the first layer. There we have some weight matrix \(W \in \R^{N_{0} \times N_{1}}\)  which acts on input \(x \in \R^{N_{0}}\) by \(x^T W\) and this quantity is then fed through the rectifier. Of course, a layer can act on a collection of \(m > 0\) inputs stored as the rows in a matrix \(X \in \R^{m\times N_{0}}\) where now the rectifier acts componentwise. Focusing on just one neuron \(w\), or column of \(W\), rather than viewing the matrix vector product \(X w\) as a collection of inner products \(\{x_i^T w\}_{i \in [m]}\), we can think about this as a linear combination of the columns of \(X\), namely \(\sum_{t\in[N_{0}]} w_t X_t\). This elementary linear algebra observation now lends the quantization problem a rather elegant interpretation: is there some way of choosing quantized weights \(q_t\) from a fixed alphabet \(\Ac\), such as \(\{-1, 0, 1\}\), so that the walk \(Xq = \sum_{t=1}^{N_0} q_t X_t\) approximates the walk \(Xw = \sum_{t=1}^{N_0} w_t X_t\)? 
    
    As we mentioned above, the study of the existence of such a \(q\) when \(Xw = 0\) has a rather rich history from the discrepancy theory literature.  \cite{spencer1985six} in Corollary 18 was able to prove the following surprising claim. There exists an absolute constant \(c > 0\) so that given \(N\) vectors \(X_1, \hdots, X_N \in \R^{m}\) with \(\sup_{t\in[N]}\|X_t\|_{2} \leq 1\) there exists a vector \(q \in \{-1, 1\}^{N}\) so that \(\|Xw - Xq\|_{\infty} = \|Xq\|_{\infty} \leq c\log(m)\). What makes this so remarkable is that the upper bound is \textit{independent} of \(N\), or the number of vectors in the walk. Spencer further remarks that J\'anos Koml\'os has conjectured that this upper bound can be reduced to simply \(c\). The proof of the Koml\'os conjecture seems to be elusive except in special cases. One special case where it is true is if we require \(N < m\) and now allow \(q \in \{-1, 0, 1\}^N\). Theorem 16 in \cite{spencer1985six} then proves that there exists universal constants \(c \in (0,1)\) and \(K > 0\) so that for every collection of vetors \(X_1, \hdots, X_N \in \R^{m}\) with \(\max_{i \in [N]} \|X_i\|_{2} \leq 1\) there is some \(q \in \{-1, 0, 1\}^N\) with \(|\{i \in [N] : q_i = 0\}| < cN\) and \(\|Xq\|_{\infty} \leq K\). 
    
    Spencer's result inspired others to attack the Koml\'os conjecture and variants thereof. \cite{banaszczyk1990beck} was able to prove a variant of Spencer's result for vectors \(X_t\) chosen from an ellipsoid. In the special case where the ellipsoid is the unit ball in \(\R^{m}\), Banaszczyk's result says for any \(X_1, \hdots, X_N \in B_2^{m}\) there exists \(q \in \{-1, 1\}^N\) so that \(\|Xq\|_2 \leq \sqrt{m}\). This bound is tight, as it is achieved by the walk with \(N = m\) and when the vectors \(X_t\) form an orthonormal basis. Later works consider a more general notion of boundedness. Rather than controlling the infinity or Euclidean norm one might instead wonder if there exists a bit string \(q\) so that the quantized walk never leaves a sufficiently large convex set containing the origin. The first such result, to the best of our knowledge, was proven by  \cite{giannopoulos1997some}. Giannopoulos proved there that for any origin-symmetric convex set \(K \subset \R^m\) with standard Gaussian measure \(\gamma(K) \geq 1/2\) and for any collection of vectors \(X_1, \hdots, X_m \in B_2^{m}\) there exists a bit string \(q \in \{-1, 1\}^{m}\) so that \(Xq \in c \log(m)K\). Notice here that the number of vectors in this result is equal to the dimension.  \cite{banaszczyk1998balancing} strengthened this result by allowing the number of vectors to be arbitrary and further showed that, under the same assumption \(\gamma(K) \geq 1/2\), there exists a \(q \in \{-1, 1\}^N\) which satisfies \(Xq \in cK\). Scaling the hypercube appropriately, this immediately implies that the bound in Spencer's result can be reduced from \(c\log(m)\) to \(c\sqrt{1+\log(m)}\). Though the above results were formulated in the special case when \(w=0\), a result by \cite{lovasz1986discrepancy} proves that results in this special case naturally extend to results in the \textit{linear discrepancy} case when \(\|w\|_{\infty} \leq 1\), though the universal constant scales by a factor of 2.
    
     While all of these works are important contributions towards resolving the Koml{\'o}s conjecture, many important questions remain, particularly pertaining to their applicability to our problem of quantizing neural networks. Naturally the most important question remains on how to construct such a \(q\) given \(X\), \(w\). A na\"ive first guess towards answering both questions would be to solve an integer least squares problem. That is, given a data set \(X,\) a neuron \(w\), and a quantization alphabet \(\Ac\), such as \{-1, 1\}, solve
    \begin{equation}\label{eq: integer least squares}
        \begin{aligned}
        & \underset{q
        }{\text{minimize}}
        & & \|Xw - Xq\|_2^2 \\
        & \text{subject to}
        & & q_i \in \Ac, \;\; i = 1, \ldots, m.
        \end{aligned}
    \end{equation}
    It is well-known, however, that solving \eqref{eq: integer least squares} is NP-Hard. See, for example, \cite{ajtai1998shortest}.
    Nevertheless, there have been many iterative constructions of vectors \(q \in \{-1, 1\}^N\) which satisfy the bounds in the aforementioned works. A non-comprehensive list of such works includes \cite{bansal2010constructive, lovett2015constructive, rothvoss2017constructive, harvey2014discrepancy, eldan2014efficient}. Constructions of \(q\) which satisfy the bound in the result of \cite{banaszczyk1998balancing} include the works of \cite{dadush2016towards, bansal2018gram}. These works also generalize to the linear discrepancy setting. In fact, \cite{bansal2018gram} prove a much more general result which allows the use of more arbitrary alphabets other than \(\{-1, 1\}\). Their algorithm is random though, so their result holds with high probability on the execution of the algorithm. This is in contrast, as we will see, with our result which will hold with high probability on the draw of Gaussian data. Beyond this, the computational complexity of the algorithms in \cite{dadush2016towards, bansal2018gram} prohibit their use in quantizing deep neural networks. 
    %
    %
    For \cite{dadush2016towards}, this consists of looping over \(O(N_0^5)\) iterations of solving a semi-definite program and computing a Cholesky factorization for a \(N_0\times N_0\) matrix.  The Gram-Schmidt walk algorithm in \cite{bansal2018gram} has a run-time complexity of \(O(N_0 (N_0+m)^{\omega})\), where \(\omega \geq 2\) is the exponent for matrix multiplication. These complexities are already quite restrictive and only give the run-time for quantizing a single neuron. As the number of neurons in each layer is likely to be large for deep neural networks, these algorithms are simply infeasible for the task at hand. As we will see, our algorithm in comparison has a run-time complexity of \(O(N_0 m)\) per neuron which is optimal in the sense that any data driven approach towards constructing \(q\) will require one pass over the \(N_0 m\) entries of \(X\). Using a norm inequality on Banasczyzk's bound, the result in \cite{bansal2010constructive} guarantees for \(\|w\|_{\infty} \leq 1\) the existence of a \(q\) such that \(\|Xw - Xq\|_2 \leq c \sqrt{m \log(m)}\). Provided \(w\) is a generic vector in the hypercube, namely that \(\|w\|_2 \propto \sqrt{N_0}\), then a simple calculation shows that with high probability on the draw of Gaussian data \(X\) with entries having variance \(1/m\) to ensure that the columns are approximately unit norm, the Gram-Schmidt walk achieves a relative error bound of \(\frac{\|Xw - Xq\|_2}{\|Xw\|_2} \lesssim \sqrt{m \log(m)/N_0}\). As we will see in Theorem \ref{thm: first layer relative error}, our relative training error bound for quantizing neurons in the first layer decays like \(\log(N_0)\sqrt{m/N_0}\). In other words, to achieve a relative error of less than \(\eps\) in the overparametrized regime where \(N_0 \gg m\), the Gram-Schmidt walk requires on the order of \(\frac{m^3 \log^3(m)}{\eps^6}\) floating point operations as compared to our algorithm which only requires on the order of \(\frac{m^2}{\eps^2}\) floating point operations.
    
    With no quantization algorithm that is both competitive from a theoretical perspective and computationally feasible, we turn to surveying what has been done outside the mathematical realm. Perhaps the simplest manner of quantizing weights is to quantize each weight within each neuron independently. The authors in \cite{rastegari2016xnor} consider precisely this set-up in the context of convolutional neural networks (CNNs). For each weight matrix \(W^{(\ell)} \in \R^{N_{\ell} \times N_{\ell + 1}}\) the quantized weight matrix \(Q^{(\ell)}\) and optimal scaling factor \(\alpha_{\ell}\) are defined as minimizers of \(\|W^{(\ell)} - \alpha Q\|_F^2\) subject to the constraint that \(Q_{i,j} \in \{-1, 1\}\) for all \(i, j\). It turns out that the analytical solution to this optimization problem is \(Q_{i,j}^{(\ell)} = \mathrm{sign}(W_{i,j}^{(\ell)})\) and \(\alpha_{\ell} = \frac{1}{mn} \sum_{i,j} |W_{i,j}^{(\ell)}|\). This form of quantization has long been known to the digital signal processing community as Memoryless Scalar Quantization (MSQ) because it quantizes a given weight independently of all other weights. While MSQ may minimize the Euclidean distance between two weight matrices, we will see that it is far from optimal if the concern is to design a matrix \(Q\) which approximates \(W\) on an overparameterized data set. Other related approaches are investigated in, e.g., \cite{hubara2017quantized}.
    
    In a similar vein, \cite{wang2017fixed} consider learning a quantized factorization of the weight matrix \(W = XDY\), where the matrices \(X, Y\) are ternary matrices with entries \(\{-1, 0, 1\}\) and \(D\) is a full-precision diagonal matrix.  While in general this is a NP-hard problem authors use a greedy approach for constructing \(X, Y, D\) inspired by the work of \cite{kolda1998semidiscrete}. They provide simulations to demonstrate the efficacy of this method on a few pre-trained models, yet no theoretical analysis is provided for the efficacy of this framework. We would like to remark that the work  \cite{kueng2019binary} gives a framework for computing factorizations of \(W\) when \(\mathrm{rank}(W) = r\) as \(W = SA \in \R^{n\times m}\), where \(S \in \{-1,1 \}^{n \times r}, A \in \R^{r\times m}\). The reason this work is intruiging is that it does offer a means for compressing the weight matrix \(W\) by storing a smaller analog matrix \(A\) and a binarized matrix \(S\) though it does not offer nearly as much compression as if we were to replace \(W\) by a fully quantized matrix \(Q\). Indeed, the matrix \(A\) is not guaranteed to be binary or admit a representation with a low-complexity quantization alphabet. Nevertheless, \cite{kueng2019binary}  give conditions under which such a factorization exists and propose an algorithm which provably constructs \(S, A\) using semi-definite programming. They extend this analysis to the case when \(W\) is the sum of a rank \(r\) matrix and a sparse matrix but do not establish robustness of their factorization to more general noise models.
    
    Extending beyond the case where the quantization alphabet is fixed \text{a priori}, \cite{gong2014compressing} propose learning a codebook through vector quantization to quantize only the dense layers in a convolutional neural network. This stands in contrast to our work where we quantize all layers of a network. They consider clustering weights using \(k\)-means clustering and using the centroids of the clusters as the quantized weights. Moreover, they consider three different methods of clustering, which include clustering the neurons as vectors, groups of neurons thought of as sub-matrices of the weight matrix, and quantizing the neurons and the successive residuals between the cluster centers and the neurons. Beyond the fact that this work does not consider quantizing the convolutional layers, there is the additional shortcoming that clustering the neurons or groups thereof requires choosing the number of clusters in advance and choosing a maximal number of iterations to stop after. Our algorithm gives explicit control over the alphabet in advance, requires tuning only the radius of the quantization alphabet, and runs in a fixed number of iterations. Similar to the above work, we make special mention of Deep Compression by \cite{han2015deep}. Deep Compression seems to enjoy compressing large networks like AlexNet without sacrificing their empirical performance on data sets like ImageNet. There, authors consider first pruning the network and quantizing the values of the (scalar-valued) weights in a given layer using \(k\)-means clustering. This method applies both to fully connected and convolutional layers. An important drawback of this quantization procedure is that the network must be retrained, perhaps multiple times, to fine tune these learned parameters. Once these parameters have been fine tuned, the weight clusters for each layer are further compressed using Huffman coding. We further remark that quantizing in this fashion is sensitive to the initialization of the cluster weights.

\section{Algorithm and Intuition}\label{sec: algorithm and intuition}
     Going forward we will consider neural networks without bias vectors. This assumption may seem restrictive, but in practice one can always use MSQ with a big enough bit budget to control the quantization error for the bias. Even better, one may simply embed the $m$ dimensional data/activations $x$ and weights $w$ into an $m+1$ dimensional space via $x \mapsto (x,1)$ and $w \mapsto (w,b)$ so that $w^Tx +b = (w,b)^T(x,1)$. In other words, the bias term can simply be treated as an extra dimension to the weight vector, so we will henceforth ignore it. Given a trained neural network \(\Phi\) with its associated weight matrices \(W^{(\ell)}\) and a data set \(X \in \R^{m \times N_0}\), our goal is to construct quantized weight matrices \(Q^{(\ell)}\) to form a new neural network \(\widetilde{\Phi}\) for which \(\|\Phi(X) - \widetilde{\Phi}(X)\|_F\) is small. For simplicity and ease of exposition, we will focus on the extreme case where the weights are constrained to the ternary alphabet \(\{-1, 0, 1\}\), though there is no reason that our methods cannot be applied to arbitrary alphabets.
    
     Our proposed algorithm will quantize a given neuron independently of other neurons. Beyond making the analysis easier this has the practical benefit of allowing us to easily parallelize quantization across neurons in a layer. If we denote a neuron as \(w \in \R^{N_{\ell}}\), we will sucessively quantize the weights in \(w\) in a greedy data-dependent way. Let \(X \in \R^{m \times N_0}\) be our data matrix, and let \(\Phi^{(\ell-1)}, \widetilde\Phi^{(\ell-1)}\) denote the analog and quantized neural networks up to layer \(\ell-1\) respectively. 
     
     In the first layer, the aim is to achieve \( Xq = \sum\limits_{t=1}^{N_0}{q_{t} X_t }\approx \sum\limits_{t=1}^{N_0}{w_{t} {X}_t } =  Xw\)  by selecting, at the \(t^{th}\) step,  $q_{t}$ so  the running sum $\sum\limits_{j=1}^{t}{q_{j} X_j}$ tracks its analog $\sum\limits_{j=1}^{t}{w_{j} X_j}$ as well as possible in an $\ell_2$ sense. That is, at the $t^{th}$ iteration, we set
    \[ q_t := {\argmin_{p\in\{-1,0,1\}}}
    \| \sum\limits_{j=1}^{t}{w_j} X_j - \sum\limits_{j=1}^{t-1}q_{j} X_j - pX_t \|_2^2
    .\]
     It will be more amenable to analysis, and to implementation, to instead consider the equivalent dynamical system where we quantize neurons in the first layer using
    \begin{align}\label{eq: first layer dynamical system}
        u_0 :&= 0 \in \R^{m}, \nonumber\\
        q_t :&= \argmin_{p \in \{-1, 0, 1\}}\|u_{t-1} + w_t X_t - p X_t \|_2^2,\\
        u_t :&= u_{t-1} + w_t X_t - q_t X_t. \nonumber
    \end{align}
    One can see, using a simple substitution, that $u_t = \sum_{j=1}^t(w_j X_j - q_j X_j)$ is the error vector at the $t^{th}$ step. Controlling it will be a main focus in our error analysis. An interesting way of thinking about \eqref{eq: first layer dynamical system} is by imagining the analog, or unquantized, walk as a drunken walker and the quantized walk is a concerned friend chasing after them. The drunken walker can stumble with step sizes \(w_t\) along an avenue in the direction \(X_t\) but the friend can only move in steps whose lengths are encoded in the alphabet \(\Ac\).
    
    In the subsequent hidden layers, we follow a slightly modified version of \eqref{eq: first layer dynamical system}. Letting \(Y := \Phi^{(\ell-1)}(X)\), \(\widetilde{Y} := \widetilde{\Phi}^{(\ell-1)}(X) \in \R^{m \times N_{\ell}}\), we quantize neurons in layer \(\ell\) by
    \begin{align}\label{eq: hidden layer dynamical system}
        u_0 :&= 0 \in \R^{m}, \nonumber\\
        q_t :&= \argmin_{p \in \{-1, 0, 1\}}\|u_{t-1} + w_t Y_t - p \widetilde{Y}_t \|_2^2,\\
        u_t :&= u_{t-1} + w_t Y_t - q_t \widetilde{Y}_t. \nonumber
    \end{align}
    We say the vector \(q \in \R^{N_\ell}\) is the quantization of \(w\). In this work, we will provide a theoretical analysis for the behavior of \eqref{eq: first layer dynamical system} and leave analysis of \eqref{eq: hidden layer dynamical system} for future work. To that end, we re-emphasize the critical role played by \emph{the state variable} \(u_t\) defined in \eqref{eq: first layer dynamical system}. Indeed, we have the identity \(\|Xw - Xq\|_2 = \|u_{N_0}\|_2\). That is, the two neurons \(w, q\) act approximately the same on the batch of data \(X\) only provided the state variable \(\|u_{N_0}\|_2\) is well-controlled. Given bounded input \(\{(w_t, X_t)\}_{t}\), systems which admit uniform upper bounds on \(\|u_t\|_2\) will be referred to as \textit{stable}. When the \(X_t\) are random, and in our theoretical considerations they will be, we remark that this is a much stronger statement than proving convergence to a limiting distribution which is common, for example, in the Markov chain literature. For a broad survey of such Markov chain techniques, one may consult \cite{meyn2012markov}. The natural question remains: is the system \eqref{eq: first layer dynamical system} stable?
    Before we dive into the machinery of this dynamical system, we would like to remark that there is a concise  form solution for \(q_t\). Denote the greedy ternary quantizer by \(\Qc: \R \to \{-1, 0, 1\}\) with \[\Qc(z) = \argmin_{p \in \{-1, 0, 1\}} |z - p|.\] Then we have the following.
    
    \begin{lemma}\label{lem: q rounds w with dither} In the context of \eqref{eq: first layer dynamical system}, we have for any \(X_t \neq 0\) that	
    	\begin{align}\label{eq: nicer form}
    		q_t = \Qc\left( w_t + \frac{X_t^T u_{t-1}}{\|X_t\|_2^2}\right).
    	\end{align}
    \end{lemma}
    \begin{proof}
    	This follows simply by completing a square. Provided \(X_t \neq 0\), we have by the definition of \(q_t\)
    	\begin{align*}
    		q_t &= \argmin_{p \in \{-1, 0, 1\}} \| u_{t-1} + (w_t - p) X_t\|_2^2 = \argmin_{p \in \{-1, 0, 1\}} (w_t - p)^2 + 2(w_t-p) \frac{X_t^T u_{t-1}}{\|X_{t-1}\|_2^2} \\
    		&=  \argmin_{p \in \{-1, 0, 1\}} \left((w_t - p) +  \frac{X_t^T u_{t-1}}{\|X_{t-1}\|_2^2}\right)^2 -  \left(\frac{X_t^T u_{t-1}}{\|X_{t-1}\|_2^2}\right)^2. 
    	\end{align*}
    	Because the former term is always non-negative, it must be the case that the minimizer is \( \Qc\left( w_t + \frac{X_t^T u_{t-1}}{\|X_t\|_2^2}\right)\).
    \end{proof}
    
    Any analysis of the stability of \eqref{eq: first layer dynamical system} must necessarily take into account how the vectors \(X_t\) are distributed. Indeed, one can easily cook up examples which give rise to sequences of \(u_t\) which diverge rapidly. For the sake of illustration consider restricting our attention to the case when \(\|X_t\|_2 = 1\) for all \(t\). The triangle inequality gives us the crude upper bound
    \begin{align*}
        \|X(w-q)\|_2 \leq \sum_{t=1}^{N_0} |w_t - q_t| \|X_t\|_2 = \|w-q\|_1.
    \end{align*}
    Choosing \(q\) to minimize \(\|w-q\|_1\), or any \(p\)-norm for that matter, simply reduces back to the MSQ quantizer where the weights within \(w\) are quantized independently of one another, namely \(q_t = \Qc(w_t)\). It turns out that one can effectively attain this upper bound by adversarially choosing \(X_t\) to be orthogonal to \(u_{t-1}\) for all \(t\). Indeed, in that setting we have exactly the MSQ quantizer
    \begin{align*}
        q_t :&=\Qc\left( w_t + X_t^T u_{t-1}\right) = \Qc(w_t),\\
        u_t :&= u_{t-1} + (w_t - q_t)X_t.
    \end{align*}
    Consequentially, by repeatedly appealing to  orthogonality,
    \begin{align*}
        \|u_t\|_2^2 = \|u_{t-1} + (w_t - q_t)X_t\|_2^2 = \|u_{t-1}\|_2^2 + (w_t - q_t)^2\|X_t\|_2^2  = \sum_{j=1}^{t} (w_j - q_j)^2.
    \end{align*}
    \noindent
    Thus, for generic vectors $w$, and adversarially chosen $X_t$, the error $\|u_t\|_2$ scales like $\sqrt{t}$. Importantly, this adversarial construction requires knowledge of $u_{t-1}$ at ``time" $t$, in order to construct an orthogonal $X_t$. In that sense, this extreme case is rather contrived. In an opposite (but also contrived) extreme case, all of the \(X_t\) are equal, and therefore \(X_t\) is parallel to \(u_{t-1}\) for all \(t\), the dynamical system reduces to a first order greedy \(\Sigma\Delta\) quantizer
    \cite{}
    \begin{align}\label{eq: reduction to sigma delta}
        q_t &= \Qc\left( w_t + X_t^T u_{t-1} \right) = \Qc\left( w_t + \sum_{j=1}^{t-1} w_j - q_j \right), \nonumber\\
        u_t &= u_{t-1} + (w_t - q_t) X_t.
    \end{align}
    Here, when \(w_t \in [-1,1]\), one can show by induction that \(\|u_t\|_2 \leq 1/2\) for all \(t\), a dramatic contrast with the previous scenario. For more details on \(\Sigma\Delta\) quantization, see for example \cite{inose1962telemetering, daubechies2003approximating}.
    
    
    Recall that in the present context the signal we wish to approximate is not the neuron \(w\) itself, but rather \(Xw\). The goal therefore is not to minimize the error \(\|w-q\|_2\) but rather to minimize \(\|X(w -q)\|_2\), which by construction is the same as \(\|u_{N_0}\|_2\). Algebraically that means carefully selecting \(q\) so that \(w-q\) is in or very close to the kernel, or null-space, of the data matrix \(X\). This immediately suggests how overparameterization may lead to better quantization. Given \(m\) data samples stored as rows in \(X \in \R^{m \times N_0}\), having \(N_0 \gg m\) or alternatively having \(\dim(\mathrm{Span}\{x_1, \hdots, x_m\}) \ll N_0\) ensures that the kernel of \(X\) is large, and one may attempt to design $q$ so that the vector $w-q$ lies as close as possible to the kernel of $X$.

\section{Main Results}

    We are now ready to state our main result which shows that \eqref{eq: first layer dynamical system} is stable when the input data \(X\) are Gaussian. The proofs of the following theorems are deferred to Section \ref{sec: gory proofs}, as the proofs are quite long and require many supporting lemmata. 
    
    \begin{theorem}\label{thm: first layer relative error}
       Suppose \(X \in \R^{m \times N_0}\) has independent columns \(X_t \sim \Nc(0, \sigma^2 I_{m\times m})\), \(w \in \R^{N_0}\) is independent of \(X\) and satisfies \(w_t \in [-1,1]\) and \(\mathrm{dist}(w_t, \{-1, 0, 1\}) > \eps\) for all \(t\). Then, with probability at least \(1 - C\exp(-c m\log(N_0))\) on the draw of the data \(X\), if $q$ is selected according to \eqref{eq: first layer dynamical system} we have that
        \begin{align}\label{eq: relative training error}
            \frac{\|Xw - Xq\|_2}{\|Xw\|_2} \lesssim \frac{\sqrt{m}\log(N_0)}{\|w\|_2},
        \end{align}
        where \(C, c > 0\) are constants that depend on \(\eps\) in a manner that is made explicit in the statement of Theorem \ref{thm: first layer gory details}.
    \end{theorem}
    \begin{proof}
        Without loss of generality, we'll assume \(\sigma = 1/\sqrt{m}\) since this factor appears in both numerator and denominator of \eqref{eq: relative training error}. Theorem \ref{thm: first layer gory details} guarantees with probability at least \(1 - Ce^{-c m\log(N_0)}\) that \(\|u_{N_0}\|_2 = \|Xw - Xq\|_2 \lesssim \sqrt{m} \log(N_0)\). Using Lemma \ref{lem: gaussian norm concentration}, we have \(\|Xw\|_2 \gtrsim \|w\|_2\) with probability at least \(1 - 2\exp(-c_{norm} m)\). Combining these two results gives us the desired statement.
    \end{proof}
    \noindent
    For generic vectors \(w\) we have \(\|w\|_2 \propto \sqrt{N_0}\), so in this case Theorem \ref{thm: first layer relative error} tells us that up to logarithmic factors the relative error decays like \(\sqrt{m/N_0}\). As it stands, this result suggests that it is sufficient to have \(N_0 \gg m\) to obtain a small relative error. In Section \ref{sec: gory proofs}, we address the case where the feature data \(X_t\) lay in a \(d\)-dimensional subspace to get a bound in terms of \(d\) rather than \(m\). In other words, this suggests that the relative training error depends not on the number of training samples \(m\) but on the intrinsic dimension of the features \(d\). See Lemma \ref{lem: subspace model} for details.

    Our next result shows that the quantization error is well-controlled in the span of the training data so that the quantized weights generalize to new data.

    \begin{theorem}\label{thm: generalize}
        Define \(X, w\) and \(q\) as in the statement of Theorem \ref{thm: first layer relative error} and further suppose that \(N_0 \gg m\). Let \(X = U\Sigma V^T\) be the singular value decomposition of \(X\), and let \(z = Vg\) where \(g \sim \Nc(0, \sigma_z^2 I_{m\times m})\) is drawn independently of \(X, w\). In other words, suppose \(z\) is a Gaussian random variable drawn from the span of the training data \(x_i\). Then with probability at least \(1 - Ce^{-c m\log(N_0)} - 3\exp(-c'' m)\) we have
        \begin{align}\label{eq: generalization error}
            |z^T(w-q)|\lesssim \left(\frac{\sigma_z m}{\sigma(\sqrt{N_0} - \sqrt{m})}\right) \sigma m \log(N_0).
        \end{align}
    \end{theorem}
    \begin{proof}
        To begin, notice that the error bound in Theorem \ref{thm: first layer gory details} easily extends to the set \(X^T(B_1^{N_0}) := \{y \in \R^{N_0} : y = \sum_{i=1}^{m} a_i x_i, \,\, \|a\|_1 \leq 1\}\) with a simple yet pessimistic argument. With probability at least \(1 - Ce^{-c m\log(N_0)}\), for any $y\in X^T(B_1^{N_0})$ one has 
        \begin{align}\label{eq: convex hull generalization}
            |y^T(w-q)| &= \left| \sum_{i=1}^{m} a_i x_i^T(w-q) \right| \leq \sum_{i=1}^{m} |a_i| |x_i^T(w-q)| \nonumber \\
            &\lesssim \sum_{i=1}^{m} |a_i| \sigma m \log(N_0) \leq \sigma m \log(N_0).
        \end{align}
        Now for \(z\) as defined in the statement of this theorem define \(p := \alpha^* z\), where
        \begin{align*}
            \begin{aligned}
                \alpha^* := & \argmax_{\alpha \geq 0}
                & & \alpha \\
                & \text{subject to}
                & & \alpha z \in X^T(B_1^{N_0}).
            \end{aligned}
        \end{align*}
        If it were the case that \(\alpha^* > 0\) then we could use \eqref{eq: convex hull generalization} to get the bound
        \begin{align*}
            |z^T(w-q)| = \frac{1}{\alpha^* } \|p^T X (w-q)\|_2 \lesssim \frac{\sigma m \log(N_0)}{\alpha^*}.
        \end{align*}
        So, it behooves us to find a strictly positive lower bound on \(\alpha^*\). By the assumption that \(z = V g\), there exists \(h\in \R^m\) so that \(X^T h = z\). Since \(N_0 > m\), \(X^T\) is injective almost surely and therefore \(h\) is unique. Setting \(v := \|h\|_1^{-1} h\), observe that \(X^Tv = \|h\|_1^{-1} z\) and \(\sum_{i=1}^{m} v_i = 1\). It follows that \(\alpha^* \geq \|h\|_1^{-1}\). To lower bound \(\|h\|_1^{-1}\), note 
        \begin{align}
            \|h\|_1^{-1} \|z\|_2 &= \|X^T v\|_2 \geq \min_{\|y\|_1 = 1}\|X^T y\|_{2} \geq \left(\min_{\|\eta\|_2 = 1}\|X^T \eta\|_{2}\right) \min_{\|y\|_1 = 1} \|y\|_2 \nonumber\\
            & \gtrsim \sigma (\sqrt{N_0} - \sqrt{m}) \min_{\|y\|_1 = 1} \|y\|_2 = \frac{\sigma(\sqrt{N_0} - \sqrt{m})}{\sqrt{m}}.
        \label{eq: h1z2}\end{align}
        The penultimate inequality in the above equation follows directly from well-known bounds on the singular values of isotropic subgaussian matrices that hold with probability at least \(1-2\exp(-c'm)\) (see \citealt{vershynin2018high}). To make the argument explicit, note that \(X^T = \sigma G\), where \(G \in \R^{N_0 \times m}\) is a matrix whose rows are independent and identically distributed gaussians with \(\E[g_ig_i^T] = I_{m\times m}\) and are thus isotropic. Using Lemma \ref{lem: gaussian norm concentration} we have with probability at least \(1-\exp(-c_{norm} m /4)\) that \(\|z\|_2 = \|Vg\|_2 = \|g\|_2 \lesssim \sigma_z \sqrt{m}\). Substituting in \eqref{eq: h1z2}, we have
        \begin{align*}
          \|h\|_1^{-1} \gtrsim \frac{\sigma(\sqrt{N_0} - \sqrt{m})}{\sigma_z m}.
        \end{align*}
        Therefore, putting it all together, we have with probability at least \(1 - Ce^{-c m\log(N_0)} - 3\exp(-c'' m)\)
        \begin{align*}
            |z^T(w-q)| \lesssim \left(\frac{\sigma_z m}{\sigma(\sqrt{N_0} - \sqrt{m})}\right) \sigma m \log(N_0).
        \end{align*}
    \end{proof}
    \begin{remark}\label{rem: simplify generalization error}
        In the special case when \(\sigma_z = \sigma \sqrt{N_0/m}\), i.e. when \(\E[\|z\|_2^2 | V] = \E\|x_i\|_2^2 = \sigma^2 N_0\) and \(N_0 \gg m\), the bound in Theorem \ref{thm: generalize} reduces to
        \begin{align*}
            \frac{\sigma \sqrt{N_0 m}}{\sigma(\sqrt{N_0} - \sqrt{m})} \sigma m \log(N_0) \lesssim \sigma m^{3/2} \log(N_0).
        \end{align*}
        Furthermore, when the row data are normalized in expectation, or when \(\sigma^2 = N_0^{-1}\), this bound becomes \(\frac{m^{3/2} \log(N_0)}{\sqrt{N_0}}\).
    \end{remark}
    \begin{remark}
        Under the low-dimensional assumptions in Lemma \ref{lem: subspace model}, the bound \eqref{eq: generalization error} and the discussion in Remark \ref{rem: simplify generalization error} apply when \(m\) is replaced with \(d\).
    \end{remark}
    \begin{remark}
        The context of Theorem \ref{thm: generalize} considers the setting when the data are overparameterized, and there are fewer training data points used than the number of parameters. It is natural to wonder if better generalization bounds could be established if many training points were used to learn the quantization. In the extreme setting where \(m \gg N_0\), one could use a covering or \(\eps\)-net like argument. Specifically, if a new sample \(z\) were \(\varepsilon\) close to a training example \(x\), then \(|(z-x)^T w| \leq \|z-x\| \|w\| \lesssim \varepsilon \sqrt{N_0}.\) Such an argument could  be done easily when the number of training points is  large enough that it leads to a small $\varepsilon$. On the other hand, the curse of dimensionality stipulates that for this argument to work it would require an exponential number of training points, e.g., of order $(\frac{1}{\varepsilon})^d$ if the training data were in a \(d\)-dimensional subspace and did not exhibit any further structure.  We choose to focus on the overparametrized setting instead, but think that investigating the ``intermediate" setting, where one has more training data coming from a structured $d$-dimensional set than parameters, is an interesting avenue for future work.
    \end{remark}

Our technique for showing the stability of \eqref{eq: first layer dynamical system}, i.e.,  the boundedness of $\|X(w-q)\|_2$, relies on tools from drift analysis. Our analysis is inspired by the works of \cite{pemantle1999moment} and \cite{hajek1982hitting}. Given a real valued stochastic process \(\{Y_t\}_{t\in \N}\), those authors give conditions on the increments \(\Delta Y_t := Y_t - Y_{t-1}\) to uniformly bound the moments, or moment generating function, of the iterates \(Y_t\). These bounds can then be transformed into a bound in probability on an individual iterate \(Y_t\) using Markov's inequality. Recall that we're interested in bounding the state variable \(u_t\) induced by the system \eqref{eq: first layer dynamical system} which quantizes the first layer of a neural network. In situations like ours it is natural to analyze the increments of \(u_t\) since the innovations \((w_t, X_t)\) are jointly independent. To invoke the results of \cite{pemantle1999moment, hajek1982hitting} we'll consider the associated stochastic process \(\{\|u_t\|_2^2\}_{t\in[N_0]}\). Beyond the fact that our intent is to control the norm of the state variable, it turns out that stability analyses of vector valued stochastic processes typically involve passing the process through a real-valued and oftentimes quadratic function known as a Lyapunov function. There is a wide variety of stability theorems which require demonstrating certain properties of the image of a stochastic process under a Lyapunov function. For example, Lyapunov functions play a critical role in analyzing Markov chains as detailed in \cite{menshikov2016non}. However, there are a few details which preclude us from using one of these well-known stability results for the process \(\{\|u_t\|_2^2\}_{t\in[N_0]}\). First, even though the innovations \((w_t, X_t)\) are jointly independent the increments 
    \begin{align}\label{eq: drift first layer}
        \Delta \|u_t\|_2^2 = (w_t - q_t)^2 \|X_t\|_2^2 + 2(w_t - q_t)\langle X_t, u_{t-1} \rangle
    \end{align}
 have a dependency structure encoded by the bit sequence \(q\). In addition to this, the bigger challenge in the analysis of \eqref{eq: first layer dynamical system} is the discontinuity inherent in the definition of \(q_t\). Addressing this discontinuity in the analysis requires carefully handling the increments on the events where \(q_t\) is fixed.

    Based on our prior discussion, towards the end of Section \ref{sec: algorithm and intuition}, it would seem  that for generic data sets the stability of \eqref{eq: first layer dynamical system} lies somewhere in between the behavior of MSQ and \(\Sigma\Delta\) quantizers, and that behavior crucially depends on the ``dither" terms \(X_t^T u_{t-1}\). For the sake of analysis then, we will henceforth make the following assumptions.

    \begin{assumption}\label{assum: data independent of weights}
    The sequence \((w_t, X_t)_t\) defined on the probability space \((\Omega, \Fc, \P)\) is adapted to the filtration \(\Fc_t\). Further, all \(X_t\) and \(w_t\) are jointly independent.
    \end{assumption}
    
    \begin{assumption}\label{assum: weights uniformly bounded}
    \(\|W^{(\ell)}\|_{\infty} = \sup_{i,j} |W_{i,j}^{(\ell)}| \leq 1.\)
    \end{assumption}
    
    \noindent
    Assumption \ref{assum: data independent of weights}'s stipulation that the \(X_t\) are independent of the weights is a simplifying relaxation, and our proof technique handles the case when the \(w_t\) are deterministic. The joint independence of the \(X_t\) could be realized by splitting the global population of training data into two populations where one is used to train the analog network and another to train the quantization. In the hypotheses of Theorem \ref{thm: first layer relative error} it is also assumed that the entries of the weight vector \(w_t\) are sufficiently separated from the characters of the alphabet \(\{-1, 0, 1\}\). We want to remark that this is simply an artifact of the proof. In succinct terms, the proof strategy relies on showing that the moment generating function of the increment \(\Delta \|u_t\|_2^2\) is strictly less than \(1\) conditioned on the event that \(\|u_{t-1}\|_2\) is sufficiently large. In the extreme case where the weights are already quantized to \(\{-1, 0, 1\}\), this aforementioned event is the empty set since the state variable \(u_t\) is identically the zero function. As such, the conditioning is ill-defined. To avoid this technicality, we assume that the neural network we wish to quantize is not already quantized, namely \(\mathrm{dist}(w_t, \{-1, 0, 1\}) > \eps\) for some \(\eps > 0\) and for all \(t \in [N_{\ell}]\). The proof technique could easily be adapted to the case where the \(w_t\) are deterministic and this hypothesis is violated for \(O(1)\) weights with only minor changes to the main result, but we do not include these modifications to keep the exposition as clear as possible. 
    Assumption \ref{assum: weights uniformly bounded} is quite mild, and can be realized by scaling all neurons in a given layer by \(\|W\|_{\infty}^{-1}\). Choosing the ternary vector \(q\) according to the scaled neuron \(\|W\|_{\infty}^{-1} w\), any bound of the form \(\|X(\|W\|_{\infty}^{-1} w - q)\|_2\leq \alpha\) immediately gives the bound \(\|X(w - \|W\|_{\infty} q)\|_2 \leq \alpha \|W\|_{\infty}\). In other words, at run time the network can use the scaled ternary alphabet \(\{-\|W\|_{\infty}, 0, \|W\|_{\infty}\}\).

\section{Numerical Simulations}\label{sec: numerics}
We present three stylized examples which show how our proposed quantization algorithm affects classification accuracy on three benchmark data sets. In the following tables and figures, we'll refer to our algorithm as Greedy Path Following Quantization, or GPFQ for short. We look at classifying digits from the MNIST data set using a multilayer perceptron, classifying images from the CIFAR10 data set using a convolutional neural network, and finally looking at classifying images from the ILSVRC2012 data set, also known as ImageNet, using the VGG16 network \citep{simonyan2014very}. We trained both networks using Keras \citep{chollet2015keras} with the Tensorflow backend on a a 2020 MacBook Pro with a M1 chip and 16GB of RAM. Note that for the first two experiments our aim here is not to match state of the art results in training the unquantized neural networks. Rather, our goal is to demonstrate that given a trained neural network, our quantization algorithm yields a network that performs similarly. 
Below, we mention our design choices for the sake of completeness, and to demonstrate that our quantization algorithm does not require any special engineering beyond what is customary in neural network architectures. We have made our code available on GitHub at \url{https://github.com/elybrand/quantized_neural_networks}.

Our implementations for these simulations differ from the presentation of the theory in a few ways. First, we do not restrict ourselves to the particular ternary alphabet of \(\{-1, 0, 1\}\). In practice, it is much more useful to replace this with the equispaced alphabet \(\Ac = \alpha \times \{ -1 + \frac{2j}{M-1} : j \in \{0, 1,\hdots, M-1\}\} \subset [-\alpha, \alpha]\), where \(M\) is fixed in advance and \(\alpha\) is chosen by cross-validation. Of course, this includes the ternary alphabet \(\{-\alpha, 0, \alpha\}\) as a special case. The intuition behind choosing the alphabet's radius \(\alpha\) is to better capture the dynamic range of the true weights. For this reason we choose for every layer \(\alpha_{\ell} = C_{\alpha} \mathrm{median}(\{|W^{(\ell)}_{i, j}|\}_{i,j})\) where the constant \(C_{\alpha}\) is fixed for all layers and is chosen by cross-validation. 
Thus, the  cost associated with allowing general alphabets \(\Ac\) is storing a floating point number for each layer (i.e., $\alpha_\ell$) and \(N_{\ell} \times N_{\ell + 1}\) bit strings of length \(\log_2(2M+1)\) per layer as compared to \(N_{\ell} \times N_{\ell + 1}\) floats per layer in the unquantized setting. 

\subsection{Multilayer Perceptron with MNIST}

We trained a multilayer perceptron to classify MNIST digits ($28\times 28$ images) with two hidden layers. The first layer has 500 neurons, the second has 300 neurons, and the output layer has 10 neurons. We also used batch normalization layers \citep{ioffe2015batch} after each hidden layer and chose the ReLU activation function for all layers except the last where we used softmax. We trained the unquantized network on the full training set of \(60,000\) digits without any preprocessing. \(20\%\) of the training data was used as validation during training. We then tested on the remaining \(10,000\) images not used in the training set. We used categorical cross entropy as our loss function during training and the Adam optimizer---see \cite{kingma2014adam}---for 100 epochs with a minibatch size of 128. After training the unquantized model we used \(25,000\) samples from the training set to train the quantization. We used the same data to quantize each layer rather than splitting the data for each layer. For this experiment we restricted the alphabet to be ternary and  cross-validated over the alphabet scalar \(C_{\alpha} \in \{1, 2, \hdots, 10\}\). The results for each choice of \(C_\alpha\) are displayed in Figure \ref{fig: MNIST alphabet scalar}. As a benchmark we compared against a network quantized using MSQ, so each weight was quantized to the element of $\mathcal{A}$ that is closest to it. As we see in Figure \ref{fig: MNIST alphabet scalar}, the MSQ quantized network exhibits a high variability in its performance as a function of the alphabet scalar, whereas the GPFQ quantized network exhibits more stable behavior. Indeed, for a number of consecutive choices of  $C_\alpha$ the performance of the GPFQ quantized network was close to its unquantized counterpart. To illustrate how accuracy was affected as subsequent layers were quantized, we ran the following experiment. First, we chose the best alphabet scalar \(C_\alpha\) for each of the MSQ and GPFQ quantized networks separately. We then measured the test accuracy as each subsequent layer of the network was quantized, leaving the later ones unchanged. The median time it took to quantize a network was 288 seconds, or about 5 minutes. The results for MSQ and GPFQ are shown in Figure \ref{fig: MNIST layer accuracies}. Figure \ref{fig: MNIST layer accuracies} demonstrates that GPFQ is able to ``error correct" in the sense that quantizing a later layer can correct for errors introduced when quantizing previous ones. We also remark that in this setting we replace 32 bit floating point weights with \(\log_2(3)\) bit weights. Consequentially, we have compressed the network by a factor of approximately \(20\), and yet the drop in test accuracy for GPFQ was minimal. Further, this quick calculation assumes we use \(\log_2(3)\) bits to represent those weights which are quantized to zero. However, there are other important consequences for setting weights to zero. From a hardware perspective, the benefit is that forward propagation requires less energy due to there being fewer connections between layers. From a software perspective, multiplication by zero is an incredibly stable operation. 

\subsection{Convolutional Neural Network with CIFAR10}

Even though our theory was phrased in the language of multilayer perceptrons it is easy to rephrase it using the vocabulary of convolutional neural networks. Here, neurons are kernels and the data are patches from the full images or their feature data in the hidden layers. These patches have the same dimensions as the kernel. Matrix convolution is defined in terms of Hilbert-Schmidt inner products between the kernel and these image patches. In other words, if we were to vectorize both the kernel and the image patches then we could take the usual inner product on vectors and reduce back to the case of a multilayer perceptron. This is exactly what we do in the quantization algorithm. Since every channel of the feature data has its own kernel we quantize each channel's kernel independently.

\begin{figure}[h!tbp]
    \centering
    \begin{subfigure}{.5\textwidth}
        \centering
        \includegraphics[width=\linewidth]{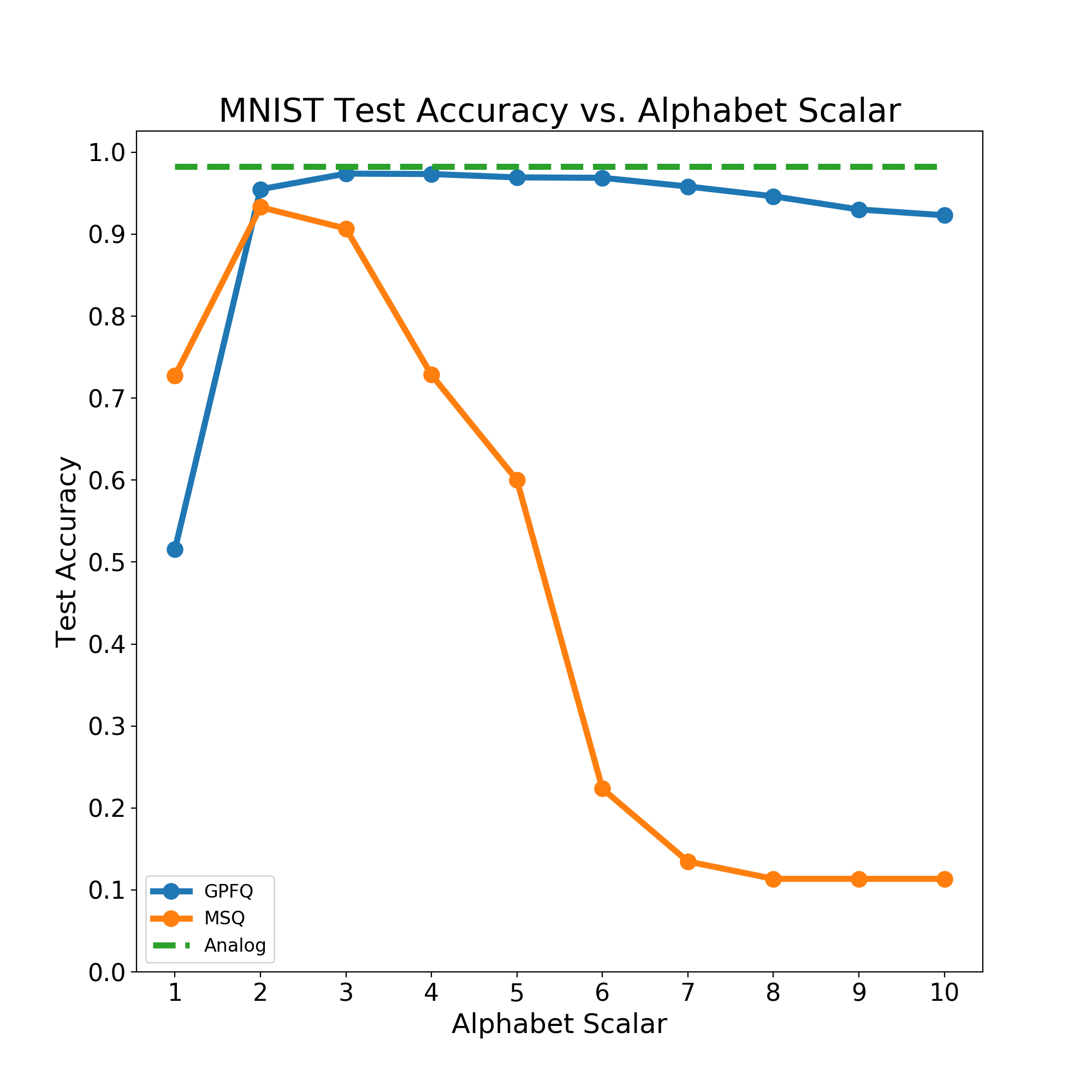}
        \caption{}
        \label{fig: MNIST alphabet scalar}
    \end{subfigure}%
    \begin{subfigure}{.5\textwidth}
         \centering
         \includegraphics[width=\linewidth]{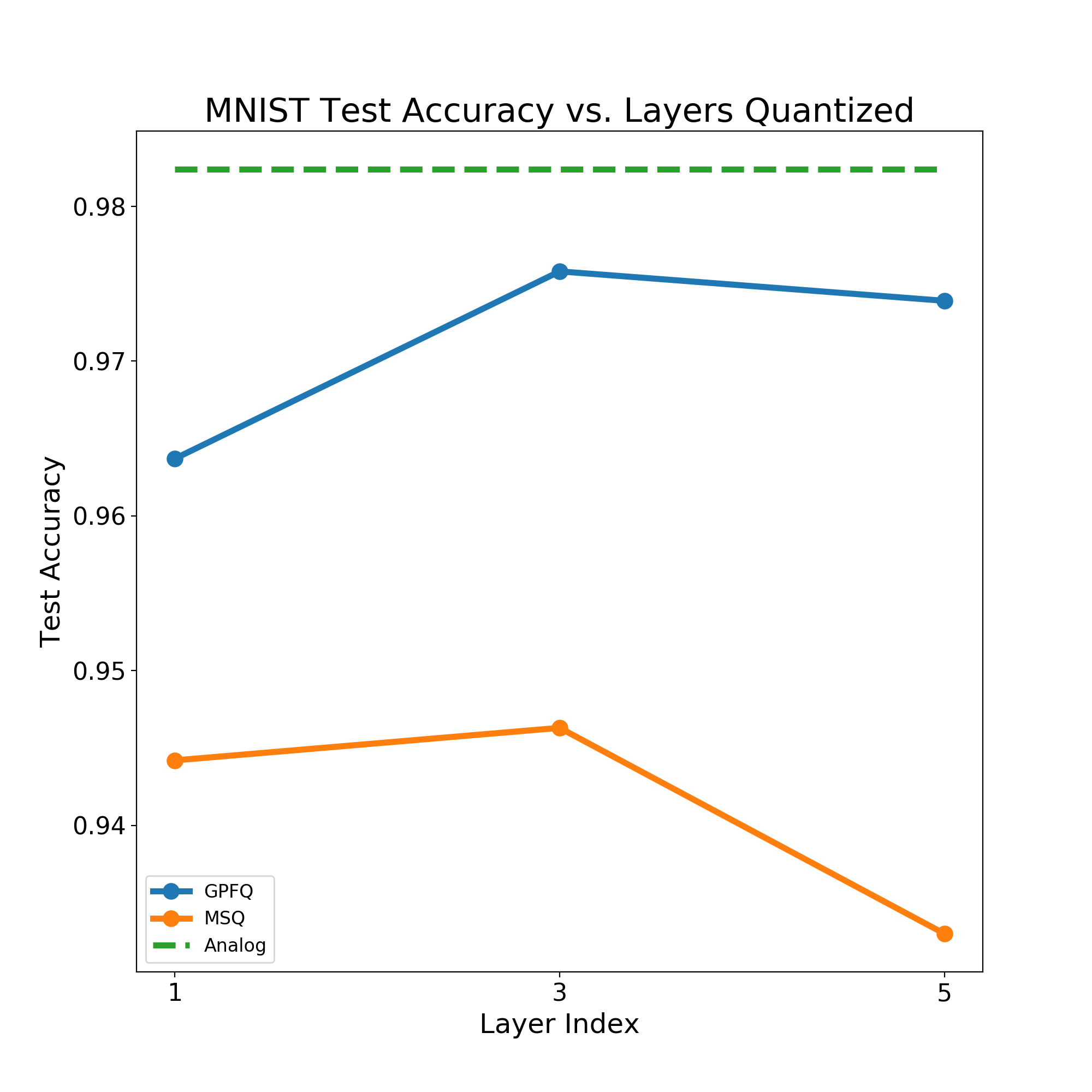}
        \caption{}
        \label{fig: MNIST layer accuracies}
    \end{subfigure}
    \caption{Comparison of GPFQ and MSQ quantized network performance on MNIST using a ternary alphabet. Figure \ref{fig: MNIST alphabet scalar} illustrates how the top-1 accuracy on the test set behaves for various alphabet scalars \(C_{\alpha}\). Figure \ref{fig: MNIST layer accuracies} demonstrates how the two quantized networks behave as each fully connected layer is successively quantized using the best alphabet scalar \(C_\alpha\) for each network. We only plot the layer indices for fully connected layers as these are the only layers we quantize.}
\end{figure}

We trained a convolutional neural network to classify images from the CIFAR10 data set with the following architecture
\begin{align*}
    2\times32C3 \to MP2 \to  2\times64C3 \to MP2 \to 2\times128C3 \to 128FC \to 10FC.
\end{align*}
\noindent
Here, \(2\times N\) C3 denotes two convolutional layers with \(N\) kernels of size \(3\times 3\), MP2 denotes a max pooling layer with kernels of size \(2\times 2\), and \(nFC\) denotes a fully connected layer with \(n\) neurons. Not listed in the above schematic are batch normalization layers which we place before every convolutional and fully connected layer except the first. During training we also use dropout layers after the max pooling layers  and before the final output layer. We use the ReLU function for every layer's activation function except the last layer where we use softmax. We preprocess the data by dividing the pixel values by \(255\) which normalizes them in the range \([0, 1]\). We augment the data set with width and height shifts as well as horizontal flips for each image. Finally, we train the network to minimize categorical cross entropy using stochastic gradient descent with a learning rate of \(10^{-4}\), momentum of \(0.9\), and a minibatch size of 64 for 400 epochs. For more information on dropout layers and pooling layers see, for example, \cite{hinton2012improving} and \cite{weng1992cresceptron}, respectively.

We trained the unquantized network on the full set of \(50,000\) training images. For training the quantization we only used the first \(5,000\) images from the training set. As we did with the multilayer perceptron on MNIST, we cross-validated the alphabet scalars \(C_{\alpha}\) over the range \(\{2, 3, 4, 5, 6\}\) and chose the best scalar for the benchmark MSQ network and the best GPFQ quantized network separately. Additionally, we cross-validated over the number of elements in the quantization alphabet, ranging over the set \(M \in \{3, 4, 8, 16\}\) which corresponds to the set of bit budgets \(\{\log_2(3), 2, 3, 4\}\). The median time it took to quantize the network using GPFQ was 1830 seconds, or about 30 minutes. The results of these experiments are shown in Table \ref{tab: CIFAR10 experiment}. In particular, the table shows that the performance of GPFQ degrades gracefully as the bit budget decreases, while the performance of MSQ drops dramatically.
\begin{table}
    \begin{tabularx}{\textwidth}{|X|X|X|X|X|}
        \multicolumn{5}{c}{CIFAR10 Top-1 Test Accuracy}\\
        \hline
        Bits          & \(C_\alpha\)     & Analog        & GPFQ    & MSQ \\
        \hline
         & 2                & 0.8922        & \bf 0.7487       & 0.1347 \\
         & 3                & 0.8922        & 0.7350       & 0.1464 \\
        \(\log_2(3)\) & 4                & 0.8922        & 0.6919       & 0.0991 \\
         & 5                & 0.8922        & 0.5627       & 0.1000 \\
         & 6                & 0.8922        & 0.3515       & 0.1000 \\\hline
                     & 2                & 0.8922        & 0.7522       & 0.2209 \\
                     & 3                & 0.8922        & \bf 0.8036       & 0.2800 \\
        2             & 4                & 0.8922        & 0.7489       & 0.1742 \\
                     & 5                & 0.8922        & 0.6748       & 0.1835 \\
                     & 6                & 0.8922        & 0.5365       & 0.1390 \\\hline
                     & 2                & 0.8922        & 0.7942       & 0.4173 \\
                     & 3                & 0.8922        & 0.8670       & 0.3754 \\
        3             & 4                & 0.8922        & \bf 0.8710       & 0.5014 \\
                     & 5                & 0.8922        & 0.8567       & 0.5652 \\
                     & 6                & 0.8922        & 0.8600       & 0.5360 \\\hline
                     & 2                & 0.8922        & 0.8124       & 0.4525 \\
                     & 3                & 0.8922        & 0.8778       & 0.7776 \\
        4             & 4                & 0.8922        & 0.8879       & 0.8443 \\
                     & 5                & 0.8922        & \bf 0.8888       & 0.8291 \\
                     & 6                & 0.8922        & 0.8810       & 0.7831 \\
        \hline
    \end{tabularx}
    \caption{This table documents the test accuracies for the analog and quantized neural networks on CIFAR10 data for the various choices of alphabet scalars \(C_\alpha\) and bit budgets.}\label{tab: CIFAR10 experiment}
\end{table}
\noindent
In this experiment, the best bit budget for both MSQ and GPFQ networks was \(4\) bits, or \(16\) characters in the alphabet. We plot the test accuracies for the best MSQ and the best GPFQ quantized network as each layer is quantized in Figure \ref{fig: CIFAR10 layer errs}. Both networks suffer from a drop in test accuracy after quantizing the second layer, but (like in the first experiment) GPFQ recovers from this dip in subsequent layers while MSQ does not. Finally, to illustrate the difference between the two sets of quantized weights in this layer we histogram the weights in Figure \ref{fig: CIFAR10 layer2 weights}.
\begin{figure}
    \centering
    \begin{subfigure}{.5\textwidth}
        \centering
        \includegraphics[width=\linewidth]{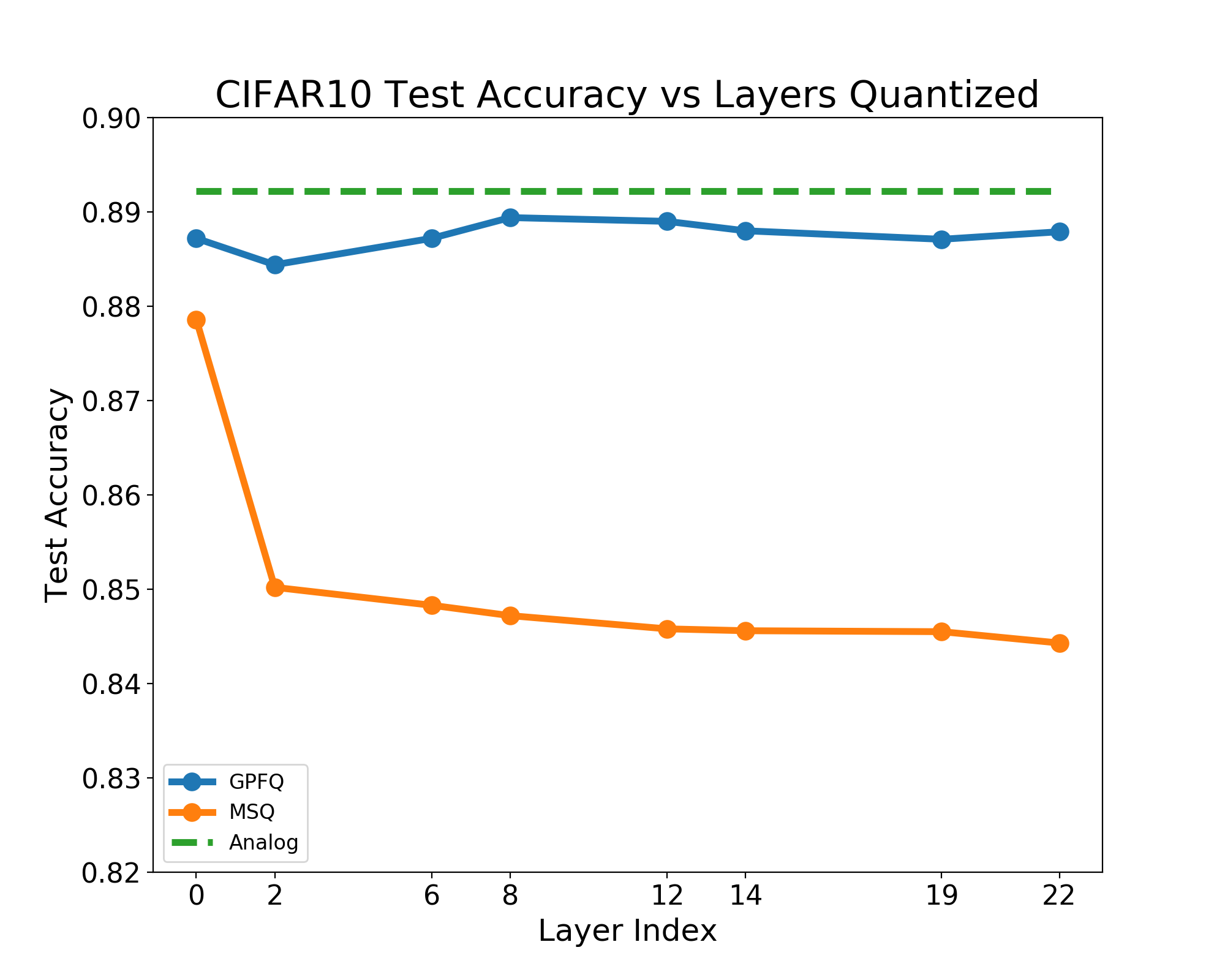}
        \caption{}
        \label{fig: CIFAR10 layer errs}
    \end{subfigure}%
    \begin{subfigure}{.5\textwidth}
         \centering
         \includegraphics[width=\linewidth]{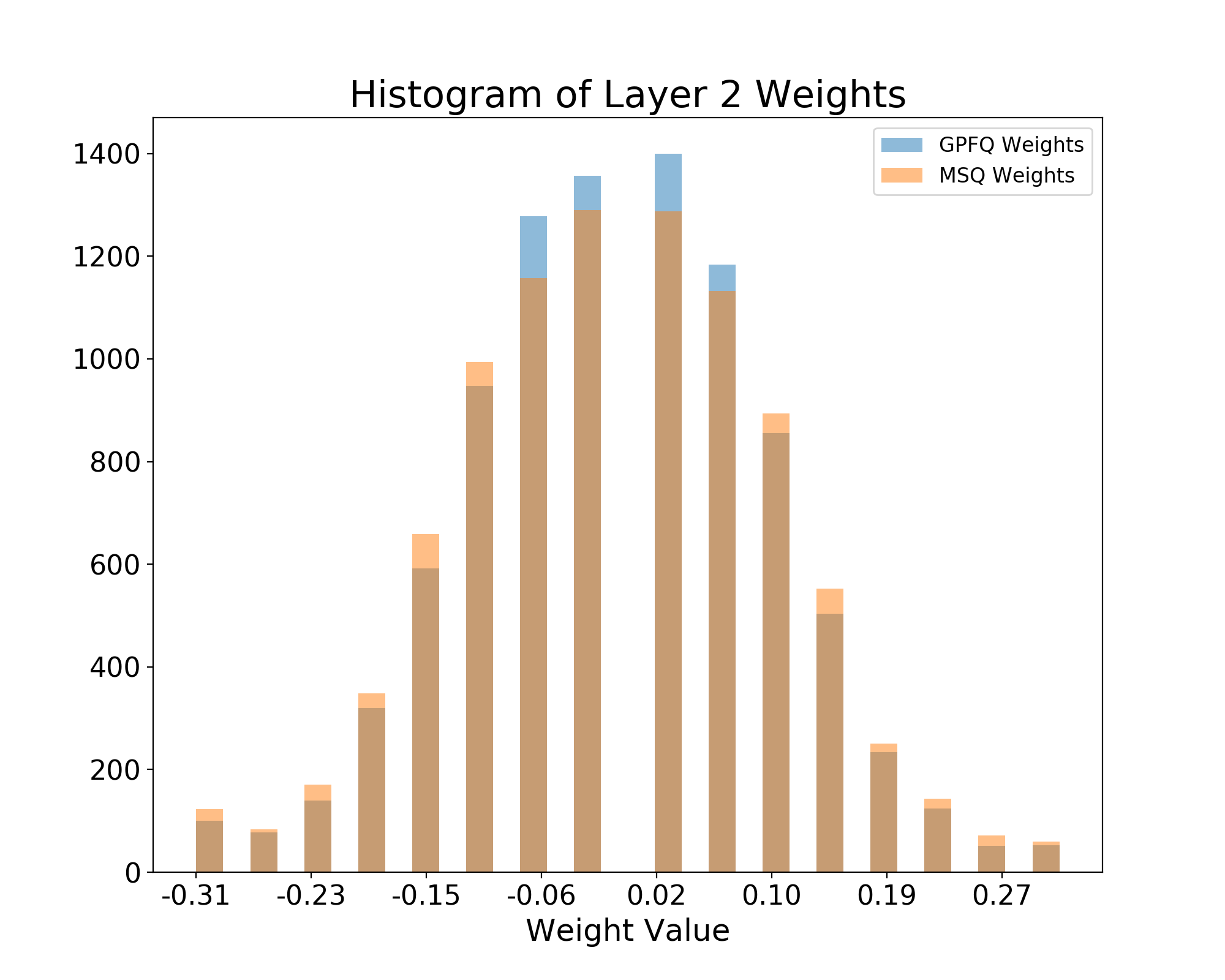}
        \caption{}
        \label{fig: CIFAR10 layer2 weights}
    \end{subfigure}
    \caption{
    Figure \ref{fig: CIFAR10 layer errs} shows how the top-1 test accuracy degrades as we quantize layers successively and leave remaining layers unquantized for the best MSQ and the best GPFQ quantized networks according to the results in Table \ref{tab: CIFAR10 experiment}. We only plot the layer indices for fully connected and convolutional layers as these are the only layers we quantize. Figure \ref{fig: CIFAR10 layer2 weights} is a  histogram of the quantized weights for the MSQ and GPFQ quantized networks at the second convolutional layer.}
\end{figure}

\subsection{VGG16 on Imagenet Data}
The previous experiments were restricted to settings where there are only 10 categories of images. To illustrate that our quantization scheme and our theory work well on more complex data sets we considered quantizing the weights of VGG16 \citep{simonyan2014very} for the purpose of classifying images from the ILSVRC2012 validation set \citep{ILSVRC15}. This data set contains 50,000 images with 1,000 categories. Since 90\% of all weights in VGG16 are in the fully connected layers, we took a similar route as \cite{gong2014compressing} and only considered quantizing the weights in the fully connected layers. We preprocessed the images in the manner that the ImageNet guidelines specify. First, we resize the smallest edge of the image to 256 pixels by using bicubic interpolation over \(4\times4\) pixel neighborhoods, and resizing the larger edge of the image to maintain the original image's aspect ratio. Next, all pixels outside the central \(224\times224\) pixels are cropped out. The image is then saved with red, green, blue (RGB) channel order\footnote{We would like to thank Caleb Robinson for outlining this procedure in his GitHub repo found at \url{https://github.com/calebrob6/imagenet_validation}.}. Finally, these processed images are further preprocessed by the function specified for VGG16 in the Keras preprocessing module. For this experiment we restrict the GPFQ quantizer to the alphabet \(\{-1, 0, 1\}\). We  cross-validate over the alphabet scalar \(C_{\alpha} \in \{2, 3, 4, 5\}\). 1500 images were randomly chosen to learn the quantization. To assess the quality of the quantized network we used 20000 randomly chosen images disjoint from the set of images used to perform the quantization and measured the top-1 and top-5 accuracy for the original VGG16 model, GPFQ, and MSQ networks. The median time it took to quantize VGG16 using GPFQ was 15391 seconds, or about 5 hours. The results from this experiment can be found in Table \ref{tab: imagenet results}. Remarkably, the best GPFQ network is able to get within \(0.65\%\) and \(0.42\%\) of the top-1 and top-5 accuracy of the analog model, respectively. In contrast, the best MSQ model can do is get within \(1.24\%\) and \(0.56\%\) of the top-1 and top-5 accuracy of the analog model, respectively. Importantly, as we saw in the previous two experiments, here again we observe a notable instability of test accuracy with respect to \(C_{\alpha}\) for the MSQ model whereas for the GPFQ model the test accuracy is more well-controlled. Moreover, just as in the CIFAR10 experiment, we see in these experiments that GPFQ networks uniformly outperform MSQ networks across quantization hyperparameter choices in both top-1 and top-5 test accuracy.
\begin{table}
    \begin{tabularx}{\textwidth}{|X|X|X|X|X|X|X|}
        \multicolumn{7}{c}{ILSVRC2012 Test Accuracy}\\
        \hline
        \(C_\alpha\) & Analog Top-1 & Analog Top-5  & GPFQ Top-1 &    GPFQ Top-5 & MSQ Top-1 & MSQ Top-5 \\
        \hline 2 & 0.7073 & 0.8977 & 0.6901 & 0.8892 & 0.68755 & 0.88785\\
        \hline 3 & 0.7073 & 0.8977 & \bf 0.70075 & \bf 0.8935 & 0.69485 & 0.8921\\
        \hline 4 & 0.7073 & 0.8977 & 0.69295 & 0.89095 & 0.66795 & 0.8713\\
        \hline 5 & 0.7073 & 0.8977 & 0.68335 & 0.88535 & 0.53855 & 0.77005\\
        \hline
    \end{tabularx}
    \caption{This table documents the test accuracy across 20000 images for the analog and quantized VGG16 networks on ILSVRC2012 data for the various choices of alphabet scalars \(C_\alpha\) using the alphabet \(\{-1, 0, 1\}\) and 1500 training images to learn the quantized weights.}\label{tab: imagenet results}
\end{table}

\section{Future Work}
Despite all of the analysis that has gone into proving stability of quantizing the first layer of a neural network using the dynamical system \eqref{eq: first layer dynamical system} and isotropic Gaussian data, there are still many interesting and unanswered questions about the performance of this quantization algorithm. The above experiments suggest that our theory can be generalized to account for non-Gaussian feature data which may have hidden dependencies between them. Beyond the subspace model we consider in Lemma \ref{lem: subspace model}, it would be interesting to extend the results to apply in the case of a manifold structure, or clustered feature data, whose intrinsic complexities can be used to improve the upper bounds in Theorem \ref{thm: first layer relative error} and Theorem \ref{thm: generalize}. Furthermore, it would be desirable to extend the analysis to address quantizing all of the hidden layers. As we showed in the experiments, our set-up naturally extends to the case of quantizing convolutional layers. Another extension of this work might consider modifying our quantization algorithm to account for other network models like recurrent networks. Finally, we observed in Theorem \ref{thm: first layer relative error} that the relative training error for learning the quantization decays like \(\log(N_0)\sqrt{m/N_0}\). We also observed in the discussion at the end of Section \ref{sec: algorithm and intuition} that when all of the feature data \(X_t\) were the same our quantization algorithm reduced to a first order greedy \(\Sigma\Delta\) quantizer. Higher order \(\Sigma\Delta\) quantizers in the context of oversampled finite frame coefficients and bandlimited functions are known to have quantization error which decays \textit{polynomially} in terms of the oversampling rate. One wonders if there exist extensions of our algorithm, perhaps with a modest increase in computational complexity, that achieve faster rates of decay for the relative quantization error. We leave all of these questions for future work.

\section{Proofs: Supporting Lemmata}

    This section  presents supporting lemmata that characterize the geometry of the dynamical system \eqref{eq: first layer dynamical system}, as well as standard results from high dimensional probability which we will use in the proof of the main technical result, Theorem \ref{thm: first layer gory details}, which appears in Section \ref{sec: gory proofs}. Outside of the high dimensional probability results, the results of Lemmas \ref{lem: q level sets in general}, \ref{lem: 1-cdf increments, pos w} and \ref{lem: 1-cdf increments, neg w} consider the behavior of the dynamical system under arbitrarily distributed data.

    \begin{lemma}\label{lem: gaussian tail bound}\cite{vershynin2018high}
        Let \(g \sim N(0,\sigma^2)\) . Then for any \(\alpha > 0\)
        \begin{align*}
            \P\left( g \geq \alpha \right) \leq \frac{\sigma}{\alpha \sqrt{2\pi}} e^{-\frac{\alpha^2}{2\sigma^2}}.
        \end{align*}
    \end{lemma}
    
    \begin{lemma}\label{lem: gaussian norm concentration}\cite{vershynin2018high}
        Let \(g \sim N(0, I_{m\times m})\) be an \(m\)-dimensional standard Gaussian vector. Then there exists some universal constant \(c_{norm} > 0\) so that for any \(\alpha > 0\)
        \begin{align*}
            \P\left( \left| \|g\|_2 - \sqrt{m} \right| \geq \alpha \right) \leq 2e^{-c_{norm} \alpha^2}.
        \end{align*}
    \end{lemma}
    

\begin{lemma}\label{lem: q level sets in general}
    Suppose that \(|w_t| < 1/2\). Then
	\begin{align*}
		\{X_t \in \R^{m}: q_t = 1\} &= B\left( \frac{1}{1-2w_t} u_{t-1},  \frac{1}{1-2w_t} \|u_{t-1}\|_2 \right),\\
		:&= B(\tilde{u}_{t-1}, \|\tilde{u}_{t-1}\|_2),\\
		\{X_t \in \R^{m}: q_t = -1\} &= B\left( \frac{-1}{1+2w_t} u_{t-1},  \frac{1}{1+2w_t} \|u_{t-1}\|_2 \right),\\
		:&= B(\hat{u}_{t-1}, \|\hat{u}_{t-1}\|_2)
	\end{align*}
\end{lemma}
\begin{proof}
    When \(q_t = 1\), \eqref{eq: nicer form} implies that
    \begin{align}\label{eq: first layer level set}
        \frac{X_t^T}{\|X_t\|_2^2}u_{t-1} \geq \frac{1}{2} - w_t 
        \iff (1-2w_t)\|X_t\|_2^2 - 2X_t^Tu_{t-1} \leq 0.
    \end{align}
    Since \(|w_t| < 1/2\), \(1-2w_t > 0\). After dividing both sides of \eqref{eq: first layer level set} by this factor, and recalling that \(\tilde{u}_{t-1} := (1-2w_t)^{-1} u_{t-1}\), we may complete the square to get the equivalent inequality
    \begin{align*}
        \left\| X_t - \tilde{u}_{t-1} \right\|_2^2 \leq  \|\tilde{u}_{t-1}\|_2^2.
    \end{align*}
    An analogous argument shows the claim for the level set \(\{X_t : q_t = -1\}\).
\end{proof}

\begin{remark}\label{re: symmetry of q_t level sets}
    When \(w > \frac{1}{2}\) or \(w < - \frac{1}{2}\), the algebra in the proof tells us that the set of \(X_t's\) for \(q_t = 1\) (resp. \(q_t = -1\)) is actually the complement of \(B\left( \tilde{u}_{t-1}, \|\tilde{u}_{t-1}\|_2 \right)\), (resp. the complement of \(B\left(\hat{u}_{t-1}, \|\hat{u}_{t-1}\|_2 \right)\)). For the special case when \(w_t = \pm 1/2\) these level sets are half-spaces.
\end{remark}

\begin{figure}
    \centering
    \includegraphics[width=0.48\textwidth]{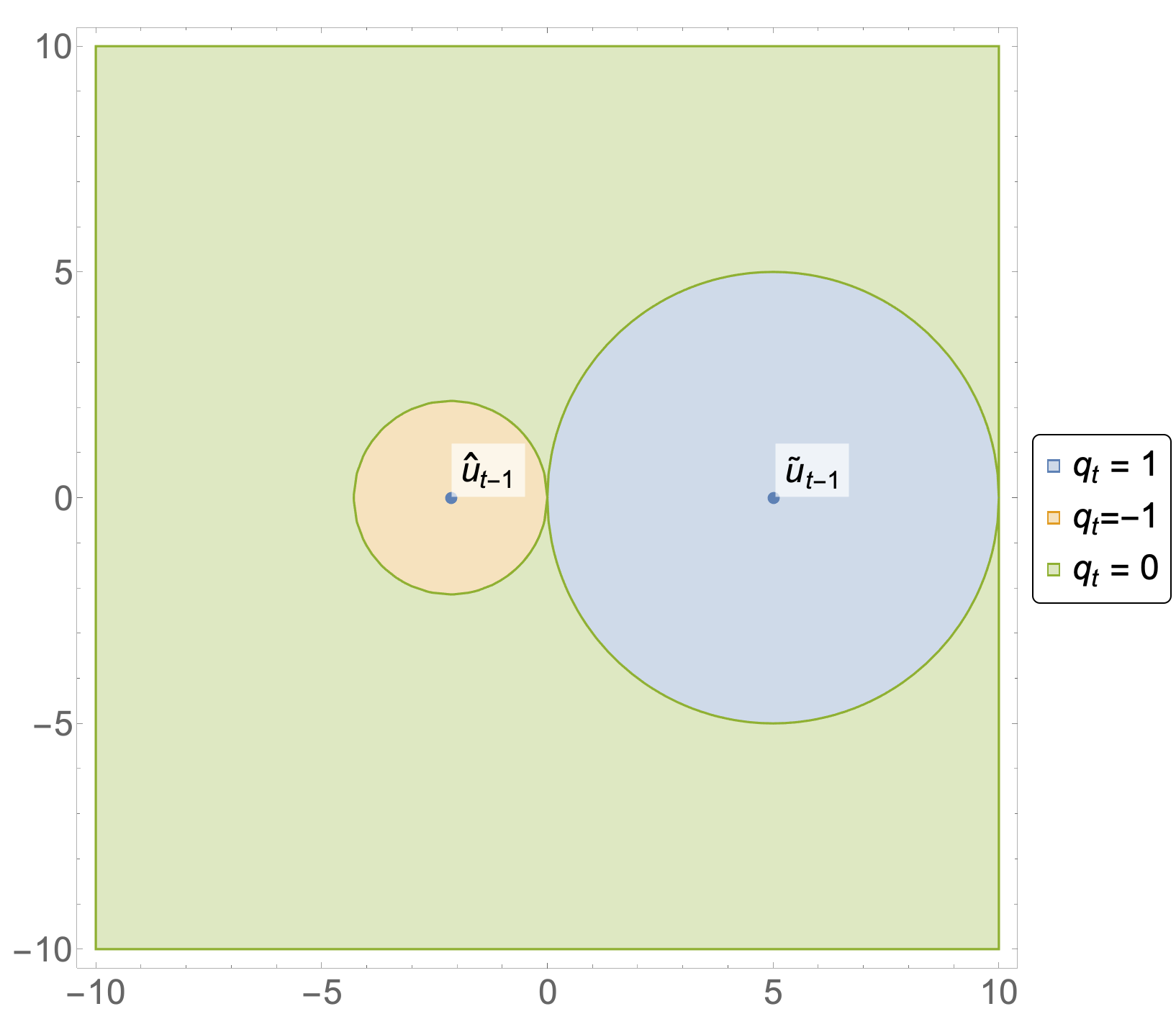}
    \hfill
    \includegraphics[width=0.48\textwidth]{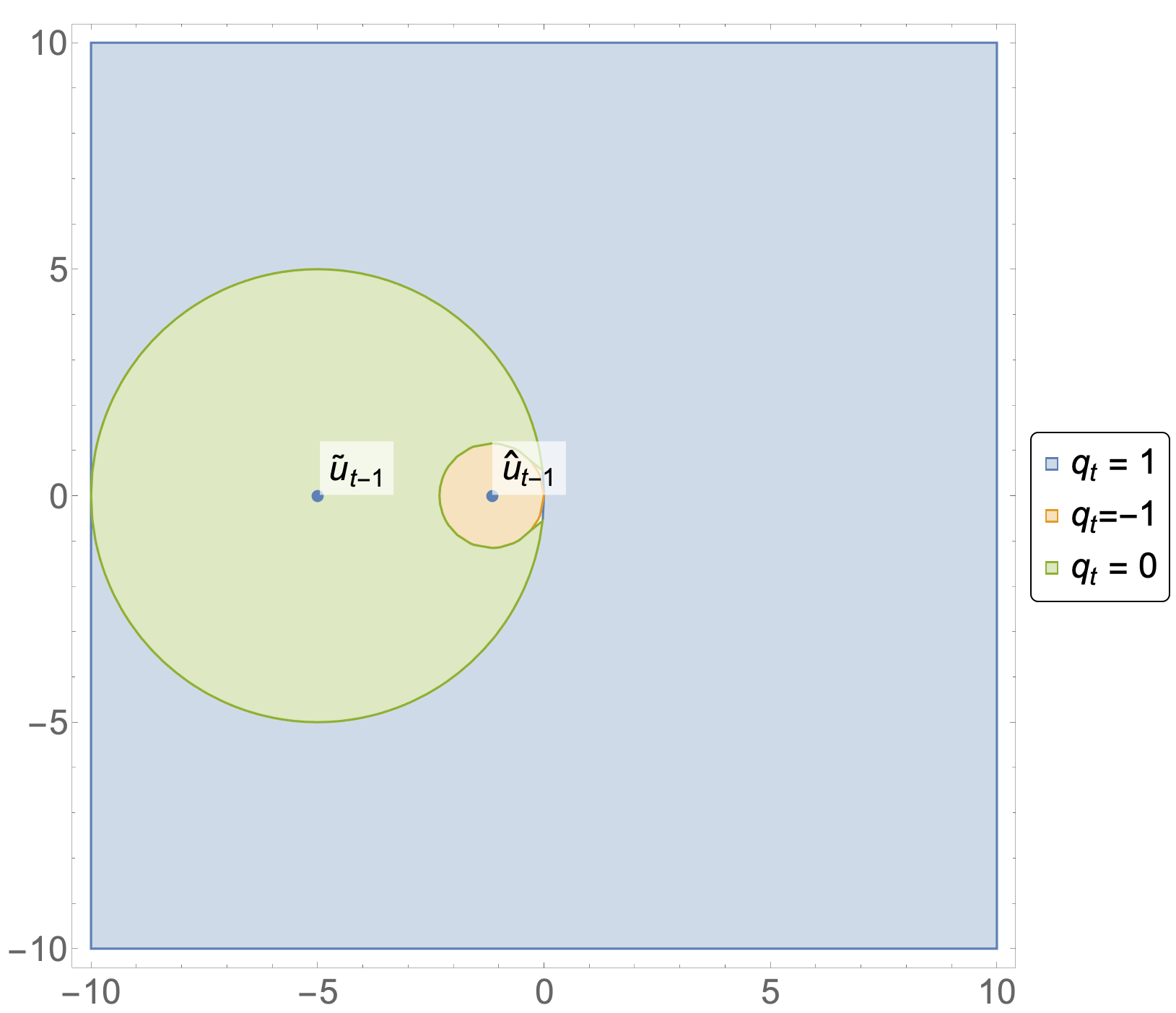}
    \caption{Visualizations of the level sets when \(u_{t-1} = 3 e_1\), \(q_t = 1\) (blue), \(q_t = -1\) (orange), and \(q_t = 0\) (green) when \(w_t = 0.2\) (left) and \(w_t = 0.8\) (right). These are the regions that must be integrated over when calculating the moment generating function of the increment \(\Delta \|u_t\|_2^2\).}
    \label{fig: regions of integration}
\end{figure}

\begin{lemma}\label{lem: 1-cdf increments, pos w} 
	Suppose \( 0 < w_t < 1\), and \(X_t\) a random vector in \(\R^m\setminus\{0\}\). Then
	\begin{align*}
		&\P_{X_t}\left((w_t - q_t)^2 + 2(w_t-q_t)\frac{X_t^T u_{t-1}}{\|X_t\|_2^2} > \alpha \Big| \Fc_{t-1} \right) \\
		& = \left\{
			\begin{array}{ll}
			      \mu_y\left( \frac{\alpha - (w_t + 1)^2}{2(w_t + 1)}, \frac{\alpha - (w_t - 1)^2}{2(w_t - 1)} \right) & \alpha < -w_t - w_t^2 \\
			      \mu_y\left( \frac{\alpha - w_t^2}{2w_t}, \frac{\alpha - (w_t - 1)^2}{2(w_t - 1)} \right) &  -w_t - w_t^2 \leq \alpha \leq w_t - w_t^2\\
			      0 & \alpha > w_t - w_t^2
			\end{array} 
			\right. 
	\end{align*}
	where \(\mu_y\) is the probability measure over \(\R\) induced by the random variable \(y := \frac{X_t^T u_{t-1}}{\|X_t\|_2^2}\).
\end{lemma}
\begin{proof}
	Let \(A_b\) denote the event that \(q_t = b\) for \(b \in \{-1, 0, 1\}\). Then by the law of total probability
	\begin{align*}
	\P\left( (w_t - q_t)^2 + 2(w_t - q_t)y > \alpha \Big| \Fc_{t-1} \right) \\
	= \sum_{b \in \{-1, 0, 1\}} \P\left( (w_t - b)^2 + 2(w_t - b)y > \alpha \text{ and } A_b \Big| \Fc_{t-1} \right).
	\end{align*}
	Therefore, we need to look at each summand in the above sum. Well, \(q_t = 0\) precisely when \(-1/2 - w_t \leq  y \leq 1/2 - w_t\). So we have
		\begin{align*}
		&\P\left( w_t^2 + 2w_t y > \alpha \text{ and } A_0 \Big| \Fc_{t-1} \right) = \P\left( y > \frac{\alpha-w_t^2}{2w_t} \text{ and } -1/2 - w_t \leq y \leq 1/2-w_t  \Big| \Fc_{t-1}\right) \\
		&=\left \{ \begin{array}{ll}
			      \mu_y\left( -1/2 - w_t, 1/2 - w_t\right) & \alpha < -w_t - w_t^2 \\
			        \mu_y\left( \frac{\alpha-w_t^2}{2w_t}, 1/2 - w_t\right) &  -w_t - w_t^2 \leq \alpha \leq w_t - w_t^2\\
			      0 & \alpha > w_t - w_t^2
			\end{array} .
			\right.  
		\end{align*}
		
		Next, \(q_t = 1\) precisely when \( y > 1/2 - w_t\). Noting that \(w_t - 1 < 0\), we have
		\begin{align*}
		&\P\left( (w_t - 1)^2 + 2(w_t - 1)y > \alpha \text{ and } A_1 \Big| \Fc_{t-1} \right) 
		\\&= \P\left( y < \frac{\alpha-(w_t - 1)^2}{2(w_t-1)} \text{ and } y > 1/2 - w_t \Big| \Fc_{t-1}  \right) \\		
		&=\left\{ \begin{array}{ll}
			        \mu_y\left( 1/2 - w_t,  \frac{\alpha - (w_t - 1)^2}{2(w_t - 1)} \right) &   \alpha \leq w_t - w_t^2\\
			      0 & \alpha > w_t - w_t^2
			\end{array} .
			\right.  		
			\end{align*}
			
		Finally, \(q_t = -1\) precisely when \( y < -1/2 - w_t\). So we have
		\begin{align*}
		&\P\left( (w_t + 1)^2 + 2(w_t + 1)y > \alpha \text{ and } A_{-1} \Big|\Fc_{t-1}\right) 
		\\&= \P\left( y > \frac{\alpha-(w_t + 1)^2}{2(w_t+1)} \text{ and } y < -1/2 - w_t \Big|\Fc_{t-1}  \right) \\		
		&=\left\{ \begin{array}{ll}
			        \mu_y\left( \frac{\alpha - (w_t + 1)^2}{2(w_t + 1)}, -1/2 - w_t \right) &   \alpha \leq -w_t - w_t^2\\
			      0 & \alpha > -w_t - w_t^2
			\end{array} .
			\right.  		
			\end{align*}
 	Summing these three piecewise functions yields the result.
\end{proof}

\begin{lemma}\label{lem: 1-cdf increments, neg w}
	When \(-1 < w_t < 0\), we have
	\begin{align*}
		&\P_{X_t}\left((w_t - q_t)^2 + 2(w_t-q_t)\frac{X_t^T u_{t-1}}{\|X_t\|_2^2} > \alpha \Big| \Fc_{t-1} \right) \\
		& = \left\{
			\begin{array}{ll}
			      \mu_y\left( \frac{\alpha - (w_t + 1)^2}{2(w_t + 1)}, \frac{\alpha - (w_t - 1)^2}{2(w_t - 1)} \right) & \alpha < w_t - w_t^2 \\
			      \mu_y\left( \frac{\alpha - (w_t + 1)^2}{2(w_t + 1)}, \frac{\alpha - w_t^2}{2w_t} \right) &  w_t - w_t^2 \leq \alpha \leq -w_t - w_t^2\\
			      0 & \alpha > -w_t - w_t^2
			\end{array} .
			\right. 
	\end{align*}
\end{lemma}

\begin{corollary}\label{cor: bounded data implies bounded increments}
    If \(\|X_t\|_2^2 \leq B\) with probability \(1\), then \(\Delta\|u_{t}\|_2^2 \leq B/4\) with probability 1.
\end{corollary}
\begin{proof}
    Using \eqref{eq: first layer dynamical system}, this follows from the identity
    \begin{align*}
        \Delta \|u_{t}\|_2^2 = \|X_t\|_2^2 \left( (w_t - q_t)^2 + 2(w_t-q_t)\frac{X_t^T u_{t-1}}{\|X_t\|_2^2}\right) \leq B \left( (w_t - q_t)^2 + 2(w_t-q_t)\frac{X_t^T u_{t-1}}{\|X_t\|_2^2}\right).
    \end{align*}
    Applying Lemma \ref{lem: 1-cdf increments, pos w} (or Lemma \ref{lem: 1-cdf increments, neg w}) on the latter quantity with $\alpha=|w_t|-w_t^2$ and recognizing that \(|w_t| - w_t^2 \leq 1/4\) when \(w_t \in [-1,1]\) yields the claim.
\end{proof}
    
\section{Proofs: Core Lemmata}\label{sec: gory proofs}
We start by proving our main result, Theorem \ref{thm: first layer gory details}, and its extension to the case where feature vectors live in a low-dimensional subspace, Lemma \ref{lem: subspace model}. The proof of Theorem \ref{thm: first layer gory details} relies on bounding the moment generating function of \(\Delta \|u_t\|_2^2 \Big| \Fc_{t-1}\), which in turn requires a number of results, referenced in the proof and presented thereafter.    These lemmas carefully deal with bounding the above moment generating function on the events where \(q_t\) is fixed. Given \(u_{t-1}\) and \(q_t = b\), Lemma \ref{lem: q level sets in general} tells us the set of directions \(X_t\) which result in \(q_t = b\) and these are the relevant events one needs to consider when bounding the moment generating function. Lemma \ref{lem: mgf q_t=0, gaussian} handles the case when \(q_t = 0\), Lemma \ref{lem: mgf q_t=1, gaussian} handles the case when \(q_t = 1\), and Lemma \ref{lem: mgf q_t=-1 gaussian} handles the case when \(q_t = -1\).

\begin{theorem}\label{thm: first layer gory details}
     
     Suppose that for $t\in \N$, the vectors \(X_t \sim \Nc(0, \sigma^2 I_{m\times m})\) are independent and that \(w_t \in [-1,1]\) are i.i.d. and independent of \(X_t\), and define the event \[A_{\eps} := \left\{\mathrm{dist}(w_t, \{-1, 0, 1\}) < \eps \right\}.\] 
     Then there exist positive constants $c_{norm}$, $C_\lambda$, and $C_{sup}$, such that with \(\lambda := \frac{C_{\lambda}}{C_{\sup}^2 \sigma^2 m\log(N_0)}\), and \(\rho, \eps \in (0,1)\) satisfying \(\tilde{\rho} := \rho + e^{C_{\lambda}/4}\P(A_\eps) < 1 \), the iteration \eqref{eq: hidden layer dynamical system} satisfies
	\begin{align}\label{eq: first layer probability bound}
		\P\left( \|u_{t}\|_2^2 > \alpha \right) \leq \tilde{\rho}^t e^{-\lambda \alpha} + \frac{1-\tilde{\rho}^{t}}{1-\tilde{\rho}} e^{\frac{C_{\lambda}}{4} + \lambda(\beta - \alpha)} + 2e^{-\log(N_0)(c_{norm}C_{\sup}^2 m - 1)}.
	\end{align}
	Above,  \(C > 0\) is a universal constant and \(\beta := C \frac{e^{8C_{\lambda}} \sigma^2 m^2 \log^2(N_0)}{\rho^2 \eps^2}\).
\end{theorem}

\begin{proof}
The proof technique is inspired by \cite{hajek1982hitting}.
Define the events
\begin{align*}
    U_{t} := \left\{\sup_{j \in \{1, \hdots, t\}} \|X_j\|_2 \leq C_{\sup} \sigma \sqrt{m}\left(\sqrt{\log(N_0)}+1\right)\right\}.
\end{align*} 
Using a union bound and Lemma \ref{lem: gaussian norm concentration}, we see that \(U_{N_0}^C\) happens with low probability since
\begin{align*}
    \P\left(\sup_{t\in [N_0]} \|X_t\|_2 > C_{\sup}\sigma \sqrt{m}\left(\sqrt{\log(N_0)}+1\right)\right) \\
    \leq 2N_0e^{-c_{norm} C_{\sup}^2 m\log(N_0)} = 2e^{-\log(N_0)(c_{norm}C_{\sup}^2 m - 1)}.
\end{align*}
We can therefore bound the probability of interest with appropriate conditioning.
\begin{align*}
    \P\left( \|u_t\|_2^2 \geq \alpha \right) 
    \leq  \P\left( \|u_t\|_2^2 \geq \alpha \Big| U_{N_0} \right)\P(U_{N_0}) + \P(U_{N_0}^C). 
\end{align*}
Looking at the first summand, for any $\lambda>0$, we have by Markov's inequality
\begin{align*}
     \P\left( \|u_t\|_2^2 \geq \alpha \Big| U_{N_0} \right)\P(U_{N_0})
     &\leq
     e^{-\lambda \alpha} \E[e^{\lambda \|u_t\|_2^2} \Big| U_{N_0} ]\P(U_{N_0})\\
     &=
     e^{-\lambda \alpha}\E[e^{\lambda \|u_t\|_2^2} \1_{U_{N_0}} ] \\
		%
		&= e^{-\lambda \alpha} \E\left[ e^{\lambda \|u_{t-1}\|_2^2} e^{\lambda \Delta\|u_{t}\|_2^2} \1_{U_{N_0}} \right]\\
		&=  e^{-\lambda \alpha} \E\left[ \E\left[e^{\lambda \|u_{t-1}\|_2^2} e^{\lambda \Delta\|u_{t}\|_2^2} \1_{U_{N_0}} \Big| \Fc_{t-1} \right]\right].
\end{align*}
We expand the conditional expectation given the filtration into a sum of two parts
\begin{align*}
    \E\left[e^{\lambda \|u_{t-1}\|_2^2} e^{\lambda \Delta\|u_{t}\|_2^2} \1_{U_{N_0}} \Big| \Fc_{t-1} \right] &=
    \E\left[e^{\lambda \|u_{t-1}\|_2^2} e^{\lambda \Delta\|u_{t}\|_2^2} \1_{U_{N_0}}  \1_{A_\eps^C \text{ and } \|u_{t-1}\|_2^2 \geq \beta} \Big| \Fc_{t-1}  \right]\\
    &+ \E\left[e^{\lambda \|u_{t-1}\|_2^2} e^{\lambda \Delta\|u_{t}\|_2^2} \1_{U_{N_0}} \1_{A_\eps \text{ or } \|u_{t-1}\|_2^2 < \beta}\Big| \Fc_{t-1}  \right].
\end{align*}
Towering expectations, the expectation over \(X_t\) of the first summand is bounded above by \(\rho e^{\lambda \|u_{t-1}\|_2^2}\1_{U_{t-1}}\) for all \(w_t\) on the event \(A_{\eps}^C\)  using Lemmas \ref{lem: mgf q_t=0, gaussian}, \ref{lem: mgf q_t=1, gaussian}, and \ref{lem: mgf q_t=-1 gaussian}. Therefore, the same bound is also true for the expectation over \(w_t\). As for the second term, we have
\begin{align*}
    &\E\left[e^{\lambda \|u_{t-1}\|_2^2} e^{\lambda \Delta\|u_{t}\|_2^2} \1_{U_{N_0}} \1_{A_\eps \text{ or } \|u_{t-1}\|_2^2 < \beta} \Big| \Fc_{t-1}  \right] =\\ 
    &\E\left[e^{\lambda \|u_{t-1}\|_2^2} e^{\lambda \Delta\|u_{t}\|_2^2} \1_{U_{N_0}} \1_{\|u_{t-1}\|_2^2 < \beta} \Big| \Fc_{t-1}  \right] + 
     \E\left[e^{\lambda \|u_{t-1}\|_2^2} e^{\lambda \Delta\|u_{t}\|_2^2}\1_{U_{N_0}} \1_{A_\eps \text{ and }   \|u_{t-1}\|_2^2 \geq \beta} \Big| \Fc_{t-1} \right]
\end{align*}
For both terms, we can use the uniform bound on the increments as proven in Corollary \ref{cor: bounded data implies bounded increments}. The first term we can bound by \(e^{\lambda C_{\sup}^2 \sigma^2 m \log(N_0)/4} e^{\lambda \beta} \leq e^{C_{\lambda}/4} e^{\lambda \beta} \). As for the second, expecting over the draw of \(w_t\) gives us 
\begin{align*}
 \E\left[e^{\lambda \|u_{t-1}\|_2^2} e^{\lambda \Delta\|u_{t}\|_2^2} \1_{U_{N_0}} \1_{A_\eps \text{ and }   \|u_{t-1}\|_2^2 \geq \beta} \Big| \Fc_{t-1} \right] &\leq e^{\lambda C_{\sup}^2 \sigma^2 m\log(N_0)/4} e^{\lambda \|u_{t-1}\|_2^2}\1_{U_{t-1}} \P( A_\eps)\\
 &\leq e^{C_{\lambda}/4} e^{\lambda \|u_{t-1}\|_2^2}\1_{U_{t-1}} \P( A_\eps).
\end{align*}
Therefore, we have
\begin{align*}
    &\P\left( \|u_{t}\|_2^2 > \alpha \right) \\
    &\leq e^{-\lambda \alpha}\left(\left(\rho + e^{C_\lambda/4} \P\left( A_\eps \right)\right)\E[e^{\lambda  \|u_{t-1}\|_2^2}\1_{U_{t-1}}] + e^{\lambda\beta + C_{\lambda}/4}\right) + 2e^{-\log(N_0)(c_{norm}C_{\sup}^2 m - 1)}\\
    &= e^{-\lambda \alpha}\left(\tilde{\rho}\E[e^{\lambda \|u_{t-1}\|_2^2}\1_{U_{t-1}}] + e^{\lambda\beta + C_{\lambda}/4}\right) + 2e^{-\log(N_0)(c_{norm}C_{\sup}^2 m - 1)}.
\end{align*}
Proceeding inductively on \(\E[e^{\lambda \|u_{t-1}\|_2^2}]\) yields the claim.
\end{proof}
\begin{remark}
    To simplify the  bound in \eqref{eq: first layer probability bound}, assuming we have \(\tilde{\rho}, \eps \propto 1\) and \(\alpha \gtrsim \beta \propto \sigma^2 m^2 \log^2(N_0)\) we have 
    \begin{align*}
        \P\left(\|u_{t}\|_2^2 \geq \alpha\right) &\leq  e^{-\lambda \alpha} + e^{\frac{C_{\lambda}}{4} + \lambda(\beta - \alpha)} + 2e^{-\log(N_0)(c_{norm}C_{\sup}^2 m - 1)},\\
        &= e^{-\frac{C_{\lambda}m \log(N_0)}{C_{\sup}^2}} + e^{\frac{C_{\lambda}}{4} - C'\frac{C_{\lambda}m \log(N_0)}{C_{\sup}^2}} + 2e^{-\log(N_0)(c_{norm}C_{\sup}^2 m - 1)},\\
        &\leq Ce^{-cm\log(N_0)}.
    \end{align*}
This matches the bound on the probability of failure we give in Theorem \ref{thm: first layer relative error}.
\end{remark}

\begin{lemma}\label{lem: subspace model}
    Suppose \(X = ZA\) where \(Z\in\R^{m\times d}\) satisfies \(Z^T Z = I\), and \(A \in \R^{d\times N_0}\) has i.i.d. \(\Nc(0, \sigma^2)\) entries. In other words, suppose the feature data \(X_t\) are Gaussians drawn from a \(d\)-dimensional subspace of \(\R^{m}\). Then with the remaining hypotheses as Theorem \ref{thm: first layer gory details} we have with probability at least \(1 - Ce^{-c d\log(N_0)} - 3\exp(-c'' d)\)
     \begin{align*}
        \|Xw - Xq\|_2 \lesssim \sigma d \log(N_0).
    \end{align*}
\end{lemma}
\begin{proof}
    We will show that running the dynamical system \eqref{eq: first layer dynamical system} with \(X_t\) is equivalent to running a modified version of \eqref{eq: first layer dynamical system} with the columns of \(A\), denoted \(A_t\). Then we can apply the result of Theorem \ref{thm: first layer gory details}. By definition, we have
    \begin{align*}
        u_0 :&= 0 \in \R^{m},\\
        q_t :&= \Qc\left(w_t + \frac{X_t^T u_{t-1}}{\|X_t\|_2^2} \right),\\
        u_t :&= u_{t-1} + w_t X_t - q_t X_t.
    \end{align*}
    In anticipation of subsequent applications of change of variables, let \(Z = U\Sigma V^T \in \R^{m\times d}\) be the singular value decomposition of \(Z\) where \(U \in \R^{m\times m}\) and \(V\in\R^{d\times d}\) are orthogonal matrices and \(\Sigma \in \R^{m\times d}\) decomposes as
    \begin{align*}
        \Sigma = \begin{bmatrix}
        I_{d\times d} \\
        0
        \end{bmatrix}.
    \end{align*}
    Since \(u_{t}\) is a linear combination of \(X_1, \hdots, X_t\) for all \(t\) it follows that \(u_t\) is in the column space of \(Z\). In other words, \(u_t = Z(Z^T Z)^{-1}Z^T u_t := Z \eta_t\). We may rewrite the above dynamical system in terms of \(A_t, \eta_t\) as
    \begin{align*}
        u_0 :&= 0 \in \R^{m},\\
        q_t :&= \Qc\left(w_t + \frac{A_t^TZ^T Z \eta_{t-1}}{\|ZA_t\|_2^2} \right)\\
        &=  \Qc\left(w_t + \frac{A_t^T \eta_{t-1}}{\|A_t\|_2^2} \right)\\
        Z\eta_t :&= Z\eta_{t-1} + w_t ZA_t - q_t ZA_t\\
        \iff \eta_t &= \eta_{t-1} + w_t A_t - q_t A_t
    \end{align*}
    So, in other words, we've reduced to running \eqref{eq: first layer dynamical system} but now with the state variables \(\eta_{t-1} \in \R^d\) in place of \(u_{t-1}\) and with \(A_t\) in place of \(X_t\). Applying the result of Theorem \ref{thm: first layer gory details} yields the claim.
\end{proof}

\begin{lemma}\label{lem: mgf q_t=0, gaussian}
    Let \(X_t \sim \Nc(0, \sigma^2 I_{m\times m})\) and \(\mathrm{dist}(w_t, \{-1, 0, 1\}) \geq \eps\). Define the event \[U := \left\{\|X_t\|_2 \leq C_{\sup} \sigma \sqrt{m\log(N_0)}\right\},\] and set \(\lambda := \frac{C_\lambda}{C_{\sup}^2 \sigma^2 m \log(N_0)}\), where \(C_\lambda \in (0, \frac{c_{norm}}{12})\) is some constant and $c_{norm}$ is as in Lemma \ref{lem: gaussian norm concentration}. Then there exists a universal constant \(C > 0\) so that with \(\beta := \frac{C e^{C_{\lambda}} \sigma^2 m \log(N_0) }{\rho \eps \sigma}\)
	\begin{align*}
		\E\left[e^{\lambda \Delta \|u_{t}\|_2^2} \1_{q_t = 0} \1_U \1_{\|u_{t-1}\|_2 \geq \beta} \Big| \Fc_{t-1} \right] \leq \rho.
	\end{align*}
\end{lemma}
\begin{proof}
    \begin{figure}
    \centering
    \includegraphics[width=0.5\linewidth]{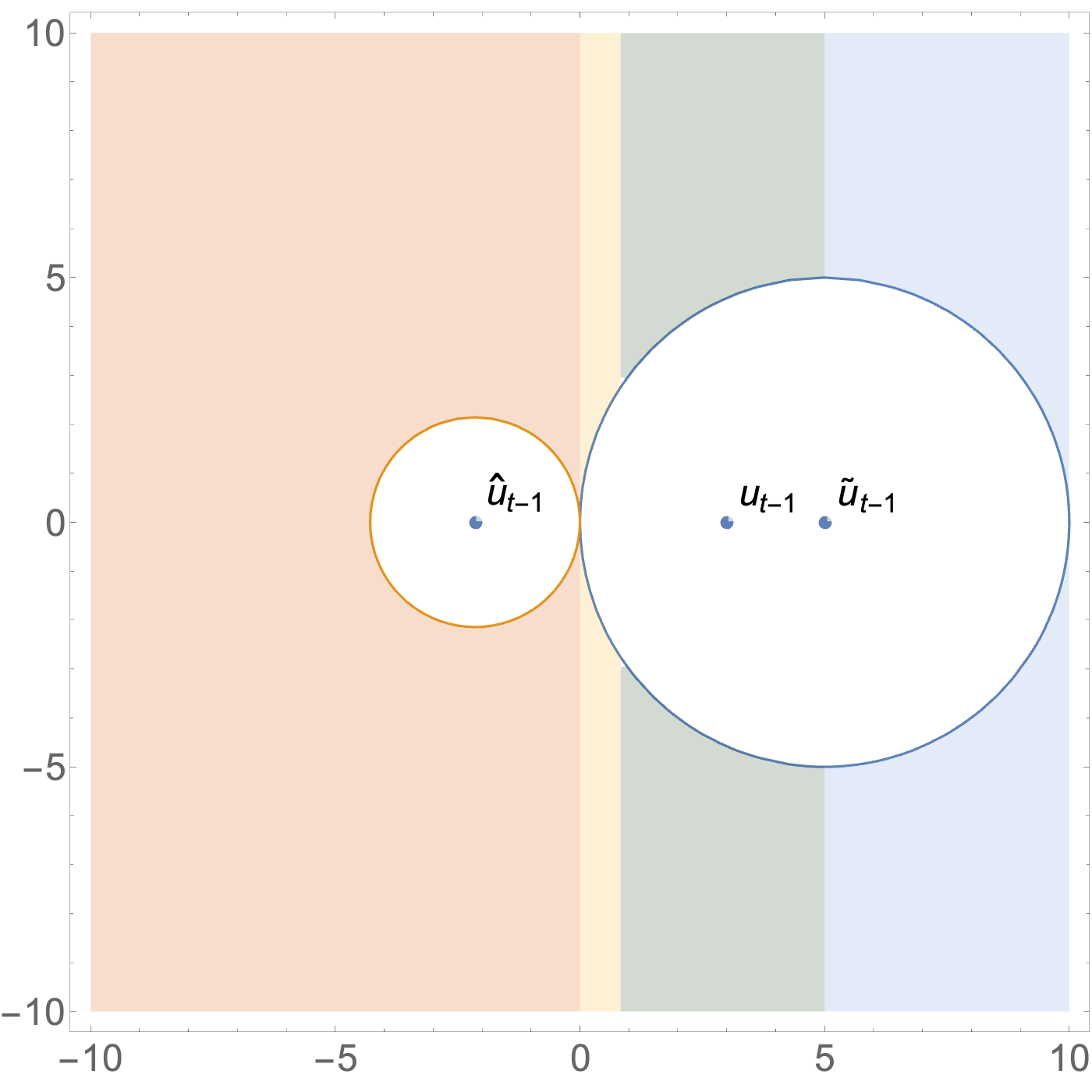}
    \caption{Plotted above is a figure depicting the various regions of integration involved in the derivation of the upper bound for Lemma \ref{lem: mgf q_t=0, gaussian} for the particular case when \(w_t = 0.3\) and \(u_t = 3 e_1\). Moving from left to right, the region in red corresponds to equation \eqref{eq: mgf qt=0 negative half space}, the region in yellow to region \(R\) as in equation \eqref{eq: mgf qt=0 region R simplified}, the region in green to region \(S\) as in equation \eqref{eq: mgf qt=0 S region, pre-simplified}, and the region in blue to region \(T\) as in equation \eqref{eq: mgf qt=0 bad half space T, simplified}.}
    \label{fig: mgf q_t=0 sketch}
    \end{figure}
    Recall \(\Delta \|u_t\|_2^2 = \|u_t\|_2^2 - \|u_{t-1}\|_2^2 = (w_t - q_t)^2 \|X_t\|_2^2 + 2(w_t - q_t)\langle X_t, u_{t-1} \rangle \).
    Let us first consider the case when \(|w_t| < \frac{1}{2}\). We will further assume that \(w_t > 0\), since there is the symmetry between \(\hat{u}_{t-1}\) and \(\tilde{u}_{t-1}\) under the mapping \(w_t \to -w_t\). Before embarking on our calculus journey, let us make some key remarks. First, on the event \(U\), we can bound the increment above by \(\Delta \|u_t\|_2^2 \leq (w_t - q_t)^2 C_{\sup}^2 \sigma^2 m\log(N_0) + 2(w_t - q_t)\langle X_t, u_{t-1} \rangle\). So, it behooves us to find an upper bound for \(\E\left[ e^{2\lambda w_t\langle X_t, u_{t-1} \rangle}\1_U\1_{q_t = 0} \1_{\|u_{t-1}\|_2 \geq \beta}| \Fc_{t-1} \right]\).  Since the exponential function is non-negative, we can always upper bound this expectation by removing the indicator on \(U\). In other words,
    \begin{align}\label{eq: simpler mgf bound qt=0}
        \E\left[ e^{2\lambda w_t\langle X_t, u_{t-1} \rangle}\1_U\1_{q_t = 0}\1_{\|u_{t-1}\|_2 \geq \beta} \Big| \Fc_{t-1} \right] \leq \E\left[ e^{2\lambda w_t\langle X_t, u_{t-1} \rangle}\1_{q_t = 0}\1_{\|u_{t-1}\|_2 \geq \beta} \Big| \Fc_{t-1} \right].
    \end{align}
    Since we're indicating on an event where \(\|u_{t-1}\|_2 \geq \beta\), we will need to handle the events where \(\langle X_t, u_{t-1} \rangle > 0\) with some care, since without an \textit{a priori} upper bound on \(\|u_{t-1}\|\) the moment generating function restricted to this event could explode. Therefore, we'll divide the region of integration into 4 pieces which are depicted in Figure \ref{fig: mgf q_t=0 sketch}. Because of the abundance of notation in the following arguments, we will denote \(\1_{\beta} := \1_{\|u_{t-1}\|_2 \geq \beta}\).
    
    Let's handle the easier event first, namely where \(\langle X_t, u_{t-1} \rangle \leq 0\).  Here, we have
    \begin{align}\label{eq: mgf qt=0 negative half space}
        &\E\left[ e^{2\lambda \langle X_t, u_{t-1} \rangle}\1_{\beta}\1_{q_t = 0} \1_{\langle X_t, u_{t-1} \rangle < 0} \Big| \Fc_{t-1} \right] = \nonumber
        \\&(2\pi\sigma^2)^{-m/2} \1_{\beta} \int_{B(\hat{u}_{t-1}, \|\hat{u}_{t-1}\|)^C\cap\{\langle x, u_{t-1} \rangle\leq 0\}} e^{2\lambda w_t \langle x, u_{t-1} \rangle} e^{\frac{-1}{2\sigma^2} \|x\|_2^2}\,\,dx.
    \end{align}
    By rotational invariance, we may assume without loss of generality that \(u_{t-1} = \|u_{t-1}\|_2 e_1\), where \(e_1 \in \R^{m}\) is the first standard basis vector. In that case, the constraint \(\langle X_t, u_{t-1} \rangle < 0\) is equivalent to \(X_{t,1} < 0\), where \(X_{t,1}\) is the first component of \(X_t\). Using Lemma \ref{lem: q level sets in general}, it follows that the set of \(X_t\) for which \(q_t = 0\) and \(X_{t,1} < 0\) is simply \(\{x \in \R^{m} : x_{1} \leq 0\} \cap B\left( -\|\hat{u}_{t-1}\|_2 e_1, \|\hat{u}_{t-1}\|_2 \right)^C\), where the negative sign here comes from the fact that \(\hat{u}_{t-1} = -(1+2w_t)u_{t-1}\). That means we can rewrite \eqref{eq: mgf qt=0 negative half space} as
    \begin{align*}
        (2\pi\sigma^2)^{-m/2} \1_{\beta} \int_{B(-\|\hat{u}_{t-1}\|_2 e_1, \|\hat{u}_{t-1}\|)^C\cap\{x_1 \leq 0\}} e^{2\lambda w_t \|u_{t-1}\| x_1 - \frac{x_1^2}{2\sigma^2}} e^{\frac{-1}{2\sigma^2} \sum_{j\geq 2}x_j^2}\,\,dx.
    \end{align*}
    Perhaps surprisingly, we can afford to use the crude upper bound on this integral by simply removing the constraint that \(x \in B(-\|\hat{u}_{t-1}\|_2 e_1, \|\hat{u}_{t-1}\|)^C\). Iterating the univariate integrals then gives us
    \begin{align}\label{eq: qt=0, easy half space simplified}
        &(2\pi\sigma^2)^{-m/2} \1_{\beta} \int_{B(-\|\hat{u}_{t-1}\|_2 e_1, \|\hat{u}_{t-1}\|)^C\cap\{x_1 \leq 0\}} e^{2\lambda w_t \|u_{t-1}\| x_1 - \frac{x_1^2}{2\sigma^2}} e^{\frac{-1}{2\sigma^2} \sum_{j\geq 2}x_j^2}\,\,dx \nonumber\\
        &\leq (2\pi\sigma^2)^{-1/2} \1_{\beta} \int_{-\infty}^{0} e^{2\lambda w_t \|u_{t-1}\| x_1 - \frac{x_1^2}{2\sigma^2}}\,\,dx_1 \int_{\R^{m-1}} (2\pi \sigma^2)^{- \frac{m-1}{2}}e^{\frac{-1}{2\sigma^2} \sum_{j\geq 2}x_j^2}\,\,dx_2\dots dx_{m}  \nonumber \\
        & = (2\pi\sigma^2)^{-1/2} \1_{\beta} \int_{-\infty}^{0} e^{2\lambda w_t \|u_{t-1}\| x_1 - \frac{x_1^2}{2\sigma^2}}\,\,dx_1
        = (2\pi\sigma^2)^{-1/2} \1_{\beta} \int_{0}^{\infty} e^{-2\lambda w_t \|u_{t-1}\| x_1 - \frac{x_1^2}{2\sigma^2}}\,\,dx_1.
    \end{align}
    We complete the square and use a change of variables to reformulate \eqref{eq: qt=0, easy half space simplified} as
    \begin{align}\label{eq: qt=0 easy half space ready for tail bound}
         &(2\pi\sigma^2)^{-1/2} \1_{\beta}e^{2\sigma^2 \lambda^2 w_t^2 \|u_{t-1}\|_2^2} \int_{0}^{\infty} e^{-\frac{1}{2\sigma^2}\left(x_1 + 2\sigma^2 \lambda w_t \|u_{t-1}\|_2 \right)^2}\,\,dx_1 \nonumber \\
         &=(2\pi\sigma^2)^{-1/2} \1_{\beta} e^{2\sigma^2 \lambda^2 w_t^2 \|u_{t-1}\|_2^2} \int_{2\sigma^2 \lambda w_t \|u_{t-1}\|_2 }^{\infty} e^{-\frac{x_1^2}{2\sigma^2}}\,\,dx_1.
    \end{align}
    Since the lower limit of integration is positive and large when \(\|u_{t-1}\|_2\) is, we can use a tail bound as in Lemma \ref{lem: gaussian tail bound} to upper bound \eqref{eq: qt=0 easy half space ready for tail bound} by
    \begin{align}\label{eq: mgf qt=0 final easy half space}
        \frac{\1_{\beta}\sigma}{2\sigma^2 \lambda w_t \|u_{t-1}\|_2 \sqrt{2\pi}} \leq \frac{\1_{\beta}\sigma}{\sigma^2 \lambda \eps \|u_{t-1}\|_2} = \frac{\1_{\beta} C_{\sup}^2 \sigma m\log(N_0)}{C_{\lambda} \eps \|u_{t-1}\|_2}, 
    \end{align}
    where the first inequality follows from \(|w_t| \geq \eps\) and the equality follows from\\ \(\lambda = \frac{C_\lambda}{C_{\sup}^2 \sigma^2 m \log(N_0)}\). 
    
    Now we handle the moment generating function on the event that \(\langle X_{t}, u_{t-1} \rangle \geq 0 \). Again, using rotational invariance to assume \(u_{t-1} = \|u_{t-1}\|_2 e_1\), we have by Lemma \ref{lem: q level sets in general} that the event to integrate over is \(\{x \in \R^{m} : x_{1} \geq 0\} \cap B\left( \|\tilde{u}_{t-1}\|_2 e_1, \|\tilde{u}_{t-1}\|_2 \right)^C\). Notice that iterating the integrals gives us
    \begin{align}\label{eq: mgf qt=0 bad halfspace, motivation for chopping}
        &\E\left[ e^{2\lambda \langle X_t, u_{t-1} \rangle}\1_{\beta} \1_{q_t = 0} \1_{\langle X_t, u_{t-1} \rangle \geq 0} \Big| \Fc_{t-1} \right]  \nonumber \\
        &\quad= (2\pi\sigma^2)^{-m/2}\1_{\beta}  \int_{B(\tilde{u}_{t-1}, \|\tilde{u}_{t-1}\|)^C\cap\{x_1 \geq 0\}} e^{2\lambda w_t \|u_{t-1}\| x_1  - \frac{1}{2\sigma^2} \|x\|_2^2}\,\,dx \nonumber \\
        &\quad= (2\pi \sigma^2)^{-1/2}\1_{\beta}  \int_{0}^{\infty} e^{2\lambda w_t \|u_{t-1}\| x_1  - \frac{x_1^2}{2\sigma^2}} \smashoperator{\int_{B\left(0, \sqrt{(2x_1\|\tilde{u}_{t-1}\|_2 - x_1^2)^+}\right)^C}} (2\pi \sigma^2)^{- \frac{m-1}{2}} e^{-\frac{1}{2\sigma^2} \sum_{j = 2}^{m} x_j^2}\,\, dx_2 \hdots dx_{m} dx_{1},
    \end{align}
        with the notation \((z)^+ = \max\{z, 0\}\) for \(z \in \R\). Consequentially, we can rephrase \eqref{eq: mgf qt=0 bad halfspace, motivation for chopping} into a more probabilistic statement. Below, let \(\gamma_j \sim \Nc(0,1)\) denote i.i.d. standard normal random variables. Then \eqref{eq: mgf qt=0 bad halfspace, motivation for chopping} is equal to
    \begin{align}\label{eq: eq: mgf qt=0 bad halfspace, probability integrand}
        (2\pi \sigma^2)^{-1/2} \1_{\beta} \int_{0}^{\infty} e^{2\lambda w_t \|u_{t-1}\| x_1  - \frac{x_1^2}{2\sigma^2}} \P\left( \sigma^2 \sum_{j=1}^{m-1} \gamma_j^2 \geq  2x_1\|\tilde{u}_{t-1}\|_2 - x_1^2 \right)\,\, dx_1.
    \end{align}
    The probability appearing in \eqref{eq: eq: mgf qt=0 bad halfspace, probability integrand} will decay exponentially provided \(2x_1\|\tilde{u}_{t-1}\|_2 - x_1^2\) is sufficiently large. To that end, we will divide up this half-space into the following regions.
    Let \(C_0 \geq 16\) be a constant and define the sets \(R := \{x \in \R^m : 0 \leq x_1 \leq \frac{C_0 \sigma^2 m}{\|\tilde{u}_{t-1}\|_2}\}\), \(S := \{x \in \R^m : \frac{C_0 \sigma^2 m}{\|\tilde{u}_{t-1}\|_2} \leq x_1 \leq \|\tilde{u}_{t-1}\|_2 \}\), and \(T := \{x \in \R^m : \|\tilde{u}_{t-1}\|_2 \leq x_1 \}\). Figure \ref{fig: mgf q_t=0 sketch} gives a visual depiction of this decomposition. Then we have
    \begin{align*}
        \E\left[ e^{2\lambda \langle X_t, u_{t-1} \rangle}\1_{\beta} \1_{q_t = 0} \1_{\langle X_t, u_{t-1} \rangle \geq 0} \Big|\Fc_{t-1} \right]  
        %
        %
        &= (2\pi\sigma^2)^{-m/2}\1_{\beta}  \smashoperator{\int_{B(\tilde{u}_{t-1}, \|\tilde{u}_{t-1}\|)^C\cap R}} e^{2\lambda w_t \|u_{t-1}\| x_1  - \frac{1}{2\sigma^2} \|x\|_2^2}\,\,dx \\
        &\quad\quad\quad\quad + (2\pi\sigma^2)^{-m/2}\1_{\beta}  \smashoperator{\int_{B(\tilde{u}_{t-1}, \|\tilde{u}_{t-1}\|)^C\cap S}} e^{2\lambda w_t \|u_{t-1}\| x_1  - \frac{1}{2\sigma^2} \|x\|_2^2}\,\,dx\\
        &\quad\quad\quad\quad + (2\pi\sigma^2)^{-m/2}\1_{\beta}  \smashoperator{\int_{B(\tilde{u}_{t-1}, \|\tilde{u}_{t-1}\|)^C\cap T}} e^{2\lambda w_t \|u_{t-1}\| x_1  - \frac{1}{2\sigma^2} \|x\|_2^2}\,\,dx.
    \end{align*}
    For the integral over \(R\), we will use the na\"ive upper bound
    \begin{align*}
        \P\left( \sigma^2 \sum_{j=1}^{m-1} \gamma_j^2 \geq  2x_1\|\tilde{u}_{t-1}\|_2 - x_1^2 \right) \leq 1.
    \end{align*}
    This gives us
    \begin{align}\label{eq: mgf qt=0 region R simplified}
        &(2\pi\sigma^2)^{-m/2} \1_{\beta} \int_{B(\tilde{u}_{t-1}, \|\tilde{u}_{t-1}\|)^C\cap R} e^{2\lambda w_t \|u_{t-1}\| x_1  - \frac{1}{2\sigma^2} \|x\|_2^2}\,\,dx \nonumber\\
        &\leq (2\pi \sigma^2)^{-1/2} \1_{\beta} \int_{0}^{\frac{C_0 \sigma^2 m}{\|\tilde{u}_{t-1}\|_2}} e^{2\lambda w_t \|u_{t-1}\| x_1  - \frac{1}{2\sigma^2} x_1^2}  \,\, dx_1 \nonumber\\
        &=  (2\pi \sigma^2)^{-1/2} \1_{\beta}  e^{2\lambda^2\sigma^2w_t^2\|u_{t-1}\|_2^2}\int_{-2\lambda w_t\sigma^2 \|u_{t-1}\|_2}^{\frac{C_0 \sigma^2 m}{\|\tilde{u}_{t-1}\|_2}-2\lambda w_t\sigma^2 \|u_{t-1}\|_2} e^{
        - \frac{1}{2\sigma^2} x_1^2}  \,\, dx_1.
    \end{align}
    The upper limit of integration is negative since \(\|u_{t-1}\|_2^2 \geq \frac{3C_0 C_{\sup}^2 \sigma^2 m \log(N_0)}{2C_{\lambda} \eps} \geq \frac{C_0 |1-2w_t|}{2\lambda w_t}\). Under this assumption, we can upper bound the integral with a Riemann sum. As the maximum of the integrand occurs at the upper limit of integration, we bound \eqref{eq: mgf qt=0 region R simplified} with
    \begin{align}\label{eq: qt=0, close to the origin}
        (2\pi \sigma^2)^{-1/2}\1_{\beta}  \frac{e^{\frac{-1}{2\sigma^2}\left(\frac{C_0^2 \sigma^4 m^2}{\|\tilde{u}_{t-1}\|_2^2} - \frac{4 C_0 \sigma^4 m \lambda w_t \|u_{t-1}\|_2}{\|\tilde{u}_{t-1}\|_2} \right)} C_0 \sigma^2 m}{\|\tilde{u}_{t-1}\|_2}.
    \end{align}
    Recognizing that \(\frac{\|u_{t-1}\|_2}{\|\tilde{u}_{t-1}\|_2} = |1-2w_t| \leq 3\) and recalling that \(\lambda = \frac{C_{\lambda}}{C_{\sup}^2 \sigma^2 m \log(N_0)}\) we can further upper bound by
    \begin{align}\label{eq: mgf qt=0 R final bound}
        \frac{\1_{\beta} e^{2C_0 \lambda \sigma^2 m w_t |1-2w_t|}C_0\sigma m}{\|\tilde{u}_{t-1}\|_2 \sqrt{2\pi}} \leq \frac{\1_{\beta} 3C_0 e^{\frac{6 C_0 C_{\lambda}}{C_{\sup}^2\log(N_0)}} \sigma m}{\|u_{t-1}\|_2}.
    \end{align}
    As was the case for \(R\), we can use the bound \(\P\left( \sigma^2 \sum_{j=1}^{m-1} \gamma_j^2 \geq  2x_1\|\tilde{u}_{t-1}\|_2 - x_1^2 \right) \leq 1\) over \(T\) too. Completing the square in the exponent as we usually do gives us
    \begin{align}\label{eq: mgf qt=0 bad half space T, simplified}
        &(2\pi\sigma^2)^{-m/2}\1_{\beta}  \int_{B(\tilde{u}_{t-1}, \|\tilde{u}_{t-1}\|)^C\cap T} e^{2\lambda w_t \|u_{t-1}\| x_1  - \frac{1}{2\sigma^2} \|x\|_2^2}\,\,dx \nonumber \\
        &\leq  (2\pi \sigma^2)^{-1/2} \1_{\beta}  e^{2\lambda^2 w_t^2 \sigma^2 \|u_{t-1}\|_2^2}\int_{\|\tilde{u}_{t-1}\|_2- 2\lambda w_t \sigma^2 \|u_{t-1}\|_2}^{\infty} e^{\frac{-x_1^2}{2\sigma^2}}\,\,dx_1.
    \end{align}
    Since \(\lambda < \frac{1}{6\sigma^2} \leq \frac{1}{2\sigma^2 w_t |1-2w_t|}\) the lower limit of integration is positive, so we can use a Gaussian tail bound as in Lemma \ref{lem: gaussian tail bound} to bound \eqref{eq: mgf qt=0 bad half space T, simplified} by
    \begin{align}\label{eq: mgf qt=0 bad half space, T first bound}
        &\frac{\1_{\beta} \sigma}{\sqrt{2\pi}\left(\|\tilde{u}_{t-1}\|_2- 2\lambda w_t \sigma^2 \|u_{t-1}\|_2\right)}e^{\frac{-1}{2\sigma^2}\left(\|\tilde{u}_{t-1}\|_2^2 - 4\lambda w_t \sigma^2 \|u_{t-1}\|_2\|\tilde{u}_{t-1}\|_2 \right)} \nonumber\\
        &=  \frac{\1_{\beta} \sigma}{\sqrt{2\pi}\left(\|\tilde{u}_{t-1}\|_2- 2\lambda w_t \sigma^2 \|u_{t-1}\|_2\right)}e^{\frac{-\|u_{t-1}\|_2^2}{2\sigma^2}\left(\frac{1}{|1-2w_t|^2} - \frac{4\lambda w_t \sigma^2}{|1-2w_t|} \right)}.
    \end{align}
    As \(\lambda < \frac{1}{12\sigma^2} \leq \frac{1}{4w_t \sigma^2 |1-2w_t|}\) the exponent appearing in \eqref{eq: mgf qt=0 bad half space, T first bound} is negative. Bounding the exponential by \(1\) then gives us the upper bound
    \begin{align}\label{eq: qt=0, hard half space, T final bound}
        \frac{\1_{\beta} \sigma}{\sqrt{2\pi}\left(\|\tilde{u}_{t-1}\|_2- 2\lambda w_t \sigma^2 \|u_{t-1}\|_2\right)}
        &= \frac{\1_{\beta} \sigma}{\|u_{t-1}\|_2 \left(\frac{1}{|1-2w_t|} - \frac{2w_t C_\lambda \sigma^2}{C_{\sup}^2 \sigma^2m\log(N_0)}\right)} \nonumber\\
        &\leq \frac{\1_{\beta} \sigma}{\|u_{t-1}\|_2 \left(\frac{1}{3} - \frac{2C_\lambda}{C_{\sup}^2 m\log(N_0)}\right)}.
    \end{align}
    Now, for \(S\) we can use the exponential decay of the probability appearing in \eqref{eq: eq: mgf qt=0 bad halfspace, probability integrand}. To make the algebra a bit nicer, we can upper-bound this probability by \(\P\left(\sigma^2\sum_{j=1}^{m-1} \gamma_j^2 \geq x_1 \|\tilde{u}_{t-1}\|\right)\) since on \(S\) we have \(0 \leq x_1 \leq \|\tilde{u}_{t-1}\|_2\). Setting \(\nu := \frac{1}{\sigma\sqrt{m-1}}\sqrt{x_1 \|\tilde{u}_{t-1}\|}\), Lemma \ref{lem: gaussian norm concentration} tells us for \(x_1 \geq \frac{C_0\sigma^{2}m}{\|\tilde{u}_{t-1}\|} \)
		\begin{align*}
		    \P\left( \sqrt{\sum_{j=1}^{m-1} \gamma_j^2} \geq \sqrt{m-1} \nu \right) \leq 2\exp(-c_{norm} (\nu-1)^2 (m-1)).
		\end{align*}
		To simplify our algebra, we remark that for any \(c > 0\),
		\begin{align*}
		    e^{-c (m-1) (z-1)^2} \leq e^{\frac{-c}{2}(m-1)z^2},
		\end{align*}
		provided \(z \geq 4\). By our choice of \(C_0\), this happens to be the case on \(S\), as \(\frac{C_0\sigma^{2}m}{\|\tilde{u}_{t-1}\|} \leq x_1 \leq \|\hat{u}_{t-1}\|\) and so
		\begin{align*}
		\nu^2 \geq \frac{x_1 \|\tilde{u}_{t-1}\|_2}{\sigma^2 m} \geq C_0.
		\end{align*}
		This gives us the upper bound on the probability
		\begin{align*}
		    \P\left( \sigma^2 \sum_{j=1}^{m-1} \gamma_j^2 \geq  2x_1\|\tilde{u}_{t-1}\|_2 - x_1^2 \right) &\leq 2\exp(-c_{norm}(m-1)\nu^2/2) \\
		    &= 2\exp\left(\frac{-c_{norm}x_1 \|\tilde{u}_{t-1}\|_2}{2\sigma^2}\right).
		\end{align*}
	    Consequentially, we can bound the integral over \(S\) as follows
		\begin{align}\label{eq: mgf qt=0 S region, pre-simplified}
		    & (2\pi \sigma^2)^{-1/2} \1_{\beta} \int_{B(\tilde{u}_{t-1}, \|\tilde{u}_{t-1}\|)^C\cap S} e^{2\lambda w_t \|u_{t-1}\| x_1  - \frac{x_1^2}{2\sigma^2}} \P\left( \sigma^2 \sum_{j=1}^{m-1} \gamma_j^2 \geq  2x_1\|\tilde{u}_{t-1}\|_2 - x_1^2 \right)\,\, dx_1 \nonumber\\
		    &\leq 2 \cdot (2\pi \sigma^2)^{-1/2} \1_{\beta}  \int_{\frac{C_0\sigma^{2}m}{\|\tilde{u}_{t-1}\|}}^{\|\tilde{u}_{t-1}\|} e^{2\lambda w_t \|u_{t-1}\|x_1 - \frac{x_1^2}{2\sigma^2} - \frac{c_{norm}x_1\|\hat{u}_{t-1}\|_2}{2\sigma^2}}\,\,dx_1 \nonumber\\
		    &= 2 \cdot (2\pi \sigma^2)^{-1/2} \1_{\beta}  \int_{\frac{C_0\sigma^{2}m}{\|\tilde{u}_{t-1}\|}}^{\|\tilde{u}_{t-1}\|} e^{\left(2\lambda w_t \|u_{t-1}\|_2 - \frac{c_{norm} \|\tilde{u}_{t-1}\|}{2\sigma^2}\right)x_1 - \frac{x_1^2}{2\sigma^2} }\,\,dx_1.
		\end{align}
		Setting \(2\zeta := c_{norm}\|\tilde{u}\|_2 - 4\lambda \sigma^2 w_t \|u_{t-1}\|\), we have that \eqref{eq: mgf qt=0 S region, pre-simplified} is equal to
		\begin{align}\label{eq: mgf qt=0 S region, simplified}
		    2 \cdot (2\pi \sigma^2)^{-1/2} \1_{\beta}  \int_{\frac{C_0\sigma^{2}m}{\|\tilde{u}_{t-1}\|}}^{\|\tilde{u}_{t-1}\|} e^{ \frac{-2\zeta x_1 }{2\sigma^2}- \frac{x_1^2}{2\sigma^2} }\,\,dx_1
		     &= 2 \cdot (2\pi \sigma^2)^{-1/2}  \1_{\beta}  e^{\frac{\zeta^2}{2\sigma^2}} \int_{\frac{C_0\sigma^{2}m}{\|\tilde{u}_{t-1}\|}}^{\|\tilde{u}_{t-1}\|} e^{\frac{-1}{2\sigma^2}(x_1 + \zeta)^2}\,\,dx_1 \nonumber\\
		     &\leq 2 \cdot (2\pi \sigma^2)^{-1/2}   \1_{\beta}  e^{\frac{\zeta^2}{2\sigma^2}} \int_{\frac{C_0\sigma^{2}m}{\|\tilde{u}_{t-1}\|}+\zeta}^{\infty} e^{\frac{-x_1^2}{2\sigma^2}}\,\,dx_1.
		\end{align}
        We remark that \(\zeta > 0\) if \(-2\lambda w_t + \frac{c_{norm}}{2|1-2w_t|\sigma^2} > 0\) which holds since \(\lambda < \frac{c_{norm}}{12\sigma^2} < \frac{c_{norm}}{4w_t|1-2w_t|\sigma^2}\). Therefore, the lower limit of integration is positive and we can use a Gaussian tail bound as in Lemma \ref{lem: gaussian tail bound} to upper bound \eqref{eq: mgf qt=0 S region, simplified} by
        \begin{align}\label{eq: mgf qt=0 S final bound}
            \frac{\1_{\beta} 2\sigma   e^{\frac{-1}{2\sigma^2}\left( \frac{C_0^2\sigma^{4}m^2}{\|\tilde{u}_{t-1}\|^2} +   2 \frac{C_0\sigma^2 m \zeta}{\|\tilde{u}_{t-1}\|_2}\right)}}{\sqrt{2\pi}\left(\frac{C_0\sigma^{2}m}{\|\tilde{u}_{t-1}\|}+\zeta\right)}
            & \leq \frac{\1_{\beta} 2\sigma   }{\sqrt{2\pi}\zeta} = \frac{\1_{\beta} 4\sigma }{\sqrt{2\pi}\|u_{t-1}\|_2\left(\frac{c_{norm}}{|1-2w_t|} - 4\lambda \sigma^2 w_t\right)}
            \\&\leq \frac{\1_{\beta} 4\sigma}{\|u_{t-1}\|_2\left(\frac{c_{norm}}{3} - \frac{4C_{\lambda}}{C_{\sup}^2m\log(N_0)}\right)}.
        \end{align}
        Putting it all together, and remembering to add back in the factor \(e^{\lambda C_{\sup}^2 \sigma^2 m \log(N_0) w_t^2} = e^{C_{\lambda}w_t^2} \leq e^{C_{\lambda}}\) we have previously ignored, we've bound \(\E\left[e^{\lambda \Delta \|u_{t}\|_2^2}\1_{\beta}  \1_{q_t = 0} \1_U \1_{\|u_{t-1}\|_2 \geq \beta}\Big| \Fc_{t-1} \right]\) from above with
        \begin{align*}
            &\underbrace{\frac{ \1_{\beta} e^{C_{\lambda}}C_{\sup}^2 \sigma m\log(N_0)}{C_{\lambda} \eps \|u_{t-1}\|_2}}_{\text{\eqref{eq: mgf qt=0 final easy half space}}}
            + \underbrace{\frac{\1_{\beta} 3C_0 e^{\frac{6 C_0 C_{\lambda}}{C_{\sup}^2\log(N_0)} + C_{\lambda}} \sigma m}{\|u_{t-1}\|_2}}_{\text{\eqref{eq: mgf qt=0 R final bound}}}
            + \underbrace{\frac{\1_{\beta} e^{C_{\lambda}} \sigma}{\|u_{t-1}\|_2 \left(\frac{1}{3} - \frac{2C_\lambda}{C_{\sup}^2 m\log(N_0)}\right)}}_{\text{\eqref{eq: qt=0, hard half space, T final bound}}}
            \\ &+  \underbrace{\frac{\1_{\beta} 4\sigma e^{C_{\lambda}}}{\|u_{t-1}\|_2\left(\frac{c_{norm}}{3} - \frac{4C_{\lambda}}{C_{\sup}^2m\log(N_0)}\right)}}_{\text{\eqref{eq: mgf qt=0 S final bound}}}
            \lesssim \frac{\1_{\beta} e^{C_{\lambda}}\sigma m \log(N_0)}{\|u_{t-1}\|_2 \eps }.
        \end{align*}
So, when \( |w_t| < 1/2\) and \(\|u_{t-1}\|_2 \geq \beta \gtrsim \frac{\sigma m \log(N_0)}{\rho \eps}\) the claim follows.
        
        Now, let's consider the case when \(w_t \geq 1/2\). Then it must be, by Lemma \ref{lem: q level sets in general}, that \(X_t \in B(\tilde{u}_{t-1}, \|\tilde{u}_{t-1}\|_2) \cap B(\hat{u}_{t-1}, \|\hat{u}_{t-1}\|_2)^C\). By non-negativity of the exponential function, we can always upper-bound the moment generating function by instead integrating over \( X_t \in  B(\hat{u}_{t-1}, \|\hat{u}_{t-1}\|_2)^C \cap \{x_1 \leq 0\}\). Pictorially, one can see this by looking at the subfigure on the right in Figure \ref{fig: regions of integration}. In this scenario, we're integrating over the region in green. The upper bound we're proposing is derived by ignoring the constraint from the blue region on the left half-space. Using this upper bound we can retrace through the steps we took to bound the integrals over \(R, S,\) and \(T\) with only minor modifications and obtain the desired result. By symmetry, an analogous approach will work for \(w_t \leq -1/2\).
\end{proof}

\begin{lemma}\label{lem: mgf q_t=1, gaussian}
    With the same hypotheses as Lemma \ref{lem: mgf q_t=0, gaussian},
	\begin{align*}
		\E\left[e^{\lambda \Delta \|u_{t}\|_2^2} \1_{\beta} \1_{q_t = 1} \1_U \1_{\|u_{t-1}\|_2 \geq \beta} \Big| \Fc_{t-1} \right] \leq \rho.
	\end{align*}
\end{lemma}
\begin{proof}
    To begin, let's consider the case when \(w_t < 1/2\).
    Recalling that \(\tilde{u}_{t-1} = \frac{1}{1-2w_t} u_{t-1}\), and arguing as we did at the beginning of the proof of Lemma \ref{lem: mgf q_t=0, gaussian}, Lemma \ref{lem: q level sets in general} tells us
    \begin{align*}
        &\E\left[e^{\lambda \Delta \|u_{t}\|_2^2} \1_{\beta} \1_{q_t = 1} \1_U\1_{\|u_{t-1}\|_2 \geq \beta} \Big| \Fc_{t-1} \right] \nonumber \\
        &\leq  (2\pi\sigma^2)^{-m/2} \1_{\beta}  e^{\lambda C_{\sup}^2 \sigma^2 m \log(N_0) (w_t - 1)^2}\smashoperator{\int_{B(\tilde{u}_{t-1}, \|\tilde{u}_{t-1}\|_2)}} e^{2\lambda (w_t - 1) x^T u_{t-1}} e^{\frac{-1}{2\sigma^2}\|x\|_2^2}\,\,dx.
    \end{align*}
    As before, we have denoted \(\1_{\beta} := \1_{\|u_{t-1}\|_2 \geq \beta}\) for conciseness. Using rotational invariance, we may assume that \(u_{t-1} = \|u_{t-1}\|_2 e_1\). Just as we did in Lemma \ref{lem: mgf q_t=0, gaussian}, expressing this integral as nested iterated integrals gives us the probabilistic formulation
    \begin{align*}
        \frac{ \1_{\beta} e^{\lambda C_{\sup}^2 \sigma^2 m \log(N_0) (w_t - 1)^2}}{\sqrt{2\pi} \sigma}\smashoperator{\int_{0}^{\quad \quad 2\|\tilde{u}_{t-1}\|_2}} e^{2\lambda (w_t - 1) \|u_{t-1}\| x_1  - \frac{1}{2\sigma^2} x_1^2} \P\left(\sigma^2 \sum_{j=1}^{m-1} \gamma_j^2 \leq 2x_1 \|\tilde{u}_{t-1}\|_2 - x_1^2 \right)\,\,dx_1,
    \end{align*}
    where, as before, the \(\gamma_j \sim \Nc(0,1)\) are i.i.d. standard normal random variables. So, consider decomposing the above integral into the following two pieces. Set \(R := \{x \in \R^m : 0 \leq x_1 \leq \frac{ C_1 \sigma^2 m}{\|\tilde{u}_{t-1}\|_2}\}\) and \(S := \{x \in \R^m : \frac{C_1 \sigma^2 m}{\|\tilde{u}_{t-1}\|_2} \leq x_1 \leq 2\|\tilde{u}_{t-1}\|_2 \}\) where \(C_1 \in (0,1)\) is a fixed constant. Then on \(R\) we have by Lemma \ref{lem: gaussian norm concentration}
    \begin{align*}
        \P&\left( \sum_{j = 1}^{m-1} g_j^2 \leq (m-1) \left( \frac{1}{\sigma^2 (m-1)} (2x_1 \|\tilde{u}_{t-1}\|_2 - x_1^2) \right) \right) \\&\leq \P\left( \sum_{j = 1}^{m-1} g_j^2 \leq (m-1) \left( \frac{1}{\sigma^2 (m-1)} (2x_1 \|\tilde{u}_{t-1}\|_2) \right) \right) 
        \leq 2e^{-c (1-C_1)^2 (m-1)}.
    \end{align*}
    Setting aside the factor \(e^{\lambda C_{\sup}^2 \sigma^2 m \log(N_0) (w_t - 1)^2}\) for the moment, we have that the integral over \(R\) is equal to
    \begin{align}\label{eq: mgf qt=1, R simplified}
        &(2\pi \sigma^2)^{-1/2}\1_{\beta}  \int_{0}^{\frac{C_1 \sigma^2 m}{\|\tilde{u}_{t-1}\|_2}} e^{2\lambda (w_t - 1) \|u_{t-1}\| x_1  - \frac{1}{2\sigma^2} x_1^2} \P\left(\sigma^2 \sum_{j=1}^{m-1} \gamma_j^2 \leq 2x_1 \|\tilde{u}_{t-1}\|_2 - x_1^2 \right)\,\,dx \nonumber\\
        &\leq (2\pi \sigma^2)^{-1/2} \1_{\beta}   2e^{-c (1-C_1)^2 (m-1)} \int_{0}^{\frac{C_1 \sigma^2 m}{\|\tilde{u}_{t-1}\|_2}} e^{2\lambda (w_t - 1) \|u_{t-1}\| x_1  - \frac{1}{2\sigma^2} x_1^2} dx_1 \nonumber\\
        &=  (2\pi \sigma^2)^{-1/2} \1_{\beta}  2e^{-c (1-C_1)^2 (m-1)} e^{2 \sigma^2 \lambda^2 (w_t-1)^2 \|u_{t-1}\|_2^2} \int_{2\lambda \sigma^2 (1- w_t)\|u_{t-1}\|_2}^{\frac{C_1 \sigma^2 m}{\|\tilde{u}_{t-1}\|_2} + 2\lambda \sigma^2 (1- w_t)\|u_{t-1}\|_2} e^{-\frac{1}{2\sigma^2} x_1^2} dx_1.
    \end{align}
    We remark that the lower limit of integration is strictly positive. Therefore, using a Riemann approximation to the integral and knowing that the maximum of the integral occurs at the lower limit of integration bounds \eqref{eq: mgf qt=1, R simplified} above by
    \begin{align}\label{eq: mgf qt=1, R final bound}
        \1_{\beta} 2e^{-c (1-C_1)^2 (m-1)} \frac{C_1 \sigma m}{\|\tilde{u}_{t-1}\|_2 \sqrt{2\pi}}.
    \end{align}
    On \(S\), we use the bound \(\P\left( \sum_{j = 1}^{m-1} g_j^2 \leq (m-1) \left( \frac{1}{\sigma^2 (m-1)} (2x_1 \|\tilde{u}_{t-1}\|_2 - x_1^2) \right)\right)\leq 1\) to get
    \begin{align}\label{eq: mgf qt=1, S simplified}
        &(2\pi \sigma^2)^{-1/2} \1_{\beta} \times \\ &\smashoperator{\int_{\quad \quad \frac{C_1 \sigma^2 m}{\|\tilde{u}_{t-1}\|_2}}^{\quad \quad \ \ 2\|\tilde{u}_{t-1}\|_2}} e^{2\lambda (w_t - 1) \|u_{t-1}\| x_1  - \frac{1}{2\sigma^2} x_1^2} \times \P\left( \sum_{j = 1}^{m-1} g_j^2 \leq (m-1) \left( \frac{1}{\sigma^2 (m-1)} (2x_1 \|\tilde{u}_{t-1}\|_2 - x_1^2) \right)\right)\,\,dx_1 \nonumber\\
        &\leq (2\pi \sigma^2)^{-1/2}\1_{\beta}  \smashoperator{\int_{\quad \quad \frac{C_1 \sigma^2 m}{\|\tilde{u}_{t-1}\|_2}}^{\quad \quad \ \  2\|\tilde{u}_{t-1}\|_2}} e^{2\lambda (w_t - 1) \|u_{t-1}\| x_1  - \frac{1}{2\sigma^2} x_1^2} dx_1 \nonumber\\
        &=    \1_{\beta} e^{2 \sigma^2 \lambda^2 (w_t-1)^2 \|u_{t-1}\|_2^2} \int_{\frac{C_1 \sigma^2 m }{\|\tilde{u}_{t-1}\|_2} + 2\lambda \sigma^2 (1-w_t) \|u_{t-1}\|_2}^{2\|\tilde{u}_{t-1}\|_2 + 2\lambda \sigma^2 (1-w_t) \|u_{t-1}\|_2} (2\pi \sigma^2)^{-1/2} e^{-\frac{1}{2\sigma^2} x_1^2} dx_1.
    \end{align}
    Since the lower limit of integration is strictly positive, we can use a Gaussian tail bound as in Lemma \ref{lem: gaussian tail bound} to upper bound \eqref{eq: mgf qt=1, S simplified} by
    \begin{align}\label{eq: mgf qt=1 S final bound}
        \frac{\1_{\beta} \sigma e^{-\frac{1}{2\sigma^2}\left( \frac{C_1^2 \sigma^4 m^2}{\|\tilde{u}_{t-1}\|_2^2} + 4 \lambda (1-w_t) C_1 \sigma^4 m \frac{\|u_{t-1}\|_2}{\|\tilde{u}_{t-1}\|_2} \right)}}{\left(\frac{C_1 \sigma^2 m }{\|\tilde{u}_{t-1}\|_2} + 2\lambda \sigma^2 (1-w_t) \|u_{t-1}\|_2\right)\sqrt{2\pi}}
        \leq \frac{ \1_{\beta} \sigma }{\sqrt{2\pi} 2\lambda \sigma^2 (1-w_t) \|u_{t-1}\|}.
    \end{align}
    To summarize, we have shown, at least when \(w_t < 1/2\), that
    \begin{align}\label{eq: mgf qt=1 final bound, wt < 1/2}
         &\E\left[e^{\lambda \Delta \|u_{t}\|_2^2} \1_{\beta} \1_{q_t = 1} \1_U \Big| \Fc_{t-1} \right] \nonumber \\
         &\leq   \underbrace{\1_{\beta} e^{\lambda C_{\sup}^2 \sigma^2 m \log(N_0) (w_t - 1)^2} 2e^{-c (1-C_1)^2 (m-1)} \frac{C_1 \sigma m}{\|\tilde{u}_{t-1}\|_2 \sqrt{2\pi}}}_{\text{\eqref{eq: mgf qt=1, R final bound}}}
         + \underbrace{\frac{\1_{\beta} e^{\lambda C_{\sup}^2 \sigma^2 m \log(N_0) (w_t - 1)^2} \sigma }{\sqrt{2\pi} 2\lambda \sigma^2 (1-w_t) \|u_{t-1}\|}}_{\text{\eqref{eq: mgf qt=1 S final bound}}} \nonumber\\
         &\leq \frac{\1_{\beta} e^{\lambda C_{\sup}^2 \sigma^2 m \log(N_0) (w_t - 1)^2}}{\|u_{t-1}\|_2}\left(\frac{2 \sigma m |1-2w_t|}{\sqrt{2\pi}}
         + \frac{ \sigma }{\sqrt{2\pi} 2\lambda \sigma^2 (1-w_t) }\right) \nonumber\\
         &\leq \frac{\1_{\beta} e^{\lambda C_{\sup}^2 \sigma^2 m \log(N_0) (w_t - 1)^2}}{\|u_{t-1}\|_2}\left(6 \sigma m
         + \frac{\sigma}{\lambda \sigma^2 \eps}\right) \nonumber\\
         &\leq \frac{\1_{\beta} e^{4 C_{\lambda}} \sigma m \log(N_0)}{\|u_{t-1}\|_2}\left(\frac{6}{\log(N_0)}
         + \frac{1}{C_\lambda \eps}\right)\nonumber\\
         & \lesssim \frac{\1_{\beta} \sigma m \log(N_0)}{\|u_{t-1}\|_2 \eps}.
    \end{align}
    Therefore, when \(\|u_{t-1}\|_2 \geq \beta \gtrsim \frac{C e^{C_{\lambda}} \sigma m \log(N_0) }{\rho \eps}\), \eqref{eq: mgf qt=1 final bound, wt < 1/2} is bounded above by \(\rho\) as desired.
    
    Now, let's consider the case when \(w_t > 1/2\). In this scenario, we can express the expectation as
    \begin{align*}
         &\E\left[e^{\lambda \Delta \|u_{t}\|_2^2} \1_{\beta} \1_{q_t = 1} \1_U \Big| \Fc_{t-1} \right] \\
         &\leq   e^{\lambda C_{\sup}^2 \sigma^2 m \log(N_0) (w_t - 1)^2} (2\pi\sigma^2)^{-m/2} \1_{\beta} \smashoperator{\int_{B(\tilde{u}_{t-1}, \|\tilde{u}_{t-1}\|_2)^C}} e^{2\lambda (w_t - 1) x^T u_{t-1}} e^{\frac{-1}{2\sigma^2}\|x\|_2^2}\,\,dx.
    \end{align*}
    Using the exact same approach as in the proof of Lemma \ref{lem: mgf q_t=0, gaussian}, we can partition the domain of integration into the following pieces:
    \begin{align*}
        \smashoperator{\int_{B(\tilde{u}_{t-1}, \|\tilde{u}_{t-1}\|_2)^C}} e^{2\lambda (w_t - 1) x^T u_{t-1}} e^{\frac{-1}{2\sigma^2}\|x\|_2^2}\,\,dx
        \quad \quad = \quad & \quad  \smashoperator{\int_{B(\tilde{u}_{t-1}, \|\tilde{u}_{t-1}\|_2)^C \cap \{x_1 \leq - \|\tilde{u}_{t-1}\|\}}} e^{2\lambda (w_t - 1) x^T u_{t-1}} e^{\frac{-1}{2\sigma^2}\|x\|_2^2}\,\,dx\\
       \quad &+  \quad \smashoperator{\int_{B(\tilde{u}_{t-1}, \|\tilde{u}_{t-1}\|_2)^C \cap \{- \|\tilde{u}_{t-1}\| \leq x_1 \leq \frac{-C \sigma^2 m}{\|\tilde{u}_{t-1}\|}\}}} e^{2\lambda (w_t - 1) x^T u_{t-1}} e^{\frac{-1}{2\sigma^2}\|x\|_2^2}\,\,dx\\
        \quad &  + \quad  \smashoperator{\int_{B(\tilde{u}_{t-1}, \|\tilde{u}_{t-1}\|_2)^C \cap \{- \frac{-C \sigma^2 m}{\|\tilde{u}_{t-1}\|} \leq x_1 \leq 0\}}} e^{2\lambda (w_t - 1) x^T u_{t-1}} e^{\frac{-1}{2\sigma^2}\|x\|_2^2}\,\,dx\\
        &+ \quad \smashoperator{\int_{B(\tilde{u}_{t-1}, \|\tilde{u}_{t-1}\|_2)^C \cap \{0 \leq x_1\}}} e^{2\lambda (w_t - 1) x^T u_{t-1}} e^{\frac{-1}{2\sigma^2}\|x\|_2^2}\,\,dx.
    \end{align*}
    The same  arguments from the proof of Lemma \ref{lem: mgf q_t=0, gaussian}  apply here with only minor modifications. Namely, an argument exactly like that given for \eqref{eq: mgf qt=0 negative half space} gives us 
    \begin{align*}
        \int_{B(\tilde{u}_{t-1}, \|\tilde{u}_{t-1}\|_2)^C \cap \{0 \leq x_1\}} e^{2\lambda (w_t - 1) x^T u_{t-1}} e^{\frac{-1}{2\sigma^2}\|x\|_2^2}\,\,dx \leq \frac{\sigma}{\lambda \sigma^2 \eps \|u_{t-1}\|}.
    \end{align*}
    Similarly, the chain of logic used to derive \eqref{eq: qt=0, hard half space, T final bound} gives us
    \begin{align*}
        \int_{B(\tilde{u}_{t-1}, \|\tilde{u}_{t-1}\|_2)^C \cap \{x_1 \leq - \|\tilde{u}_{t-1}\|\}} e^{2\lambda (w_t - 1) x^T u_{t-1}} e^{\frac{-1}{2\sigma^2}\|x\|_2^2}\,\,dx \leq \frac{\sigma}{\|u_{t-1}\|_2 \left(\frac{1}{3} - \frac{2C_\lambda}{C_{\sup}^2 m\log(N_0)}\right)}.
    \end{align*}
    Calculations for the derivation of \eqref{eq: mgf qt=0 S final bound} give us
    \begin{align*}
         \int_{B(\tilde{u}_{t-1}, \|\tilde{u}_{t-1}\|_2)^C \cap \{- \|\tilde{u}_{t-1}\| \leq x_1 \leq \frac{-C \sigma^2 m}{\|\tilde{u}_{t-1}\|}\}} e^{2\lambda (w_t - 1) x^T u_{t-1}} e^{\frac{-1}{2\sigma^2}\|x\|_2^2}\,\,dx\\
         \leq \frac{4\sigma }{\|u_{t-1}\|_2\left(\frac{c_{norm}}{3} - \frac{4C_{\lambda}}{C_{\sup}^2m\log(N_0)}\right)}.
    \end{align*}
    Finally, the same reasoning that was used to derive \eqref{eq: mgf qt=0 R final bound} gives us
    \begin{align*}
        \int_{B(\tilde{u}_{t-1}, \|\tilde{u}_{t-1}\|_2)^C \cap \{- \frac{-C \sigma^2 m}{\|\tilde{u}_{t-1}\|} \leq x_1 \leq 0\}} e^{2\lambda (w_t - 1) x^T u_{t-1}} e^{\frac{-1}{2\sigma^2}\|x\|_2^2}\,\,dx \leq  \frac{3C_0 e^{\frac{6 C_0 C_{\lambda}}{C_{\sup}^2\log(N_0)}} \sigma m}{\|u_{t-1}\|_2}.
    \end{align*}
    Following the remainder of the proof of Lemma \ref{lem: mgf q_t=0, gaussian} in this scenario gives us the result when \(w_t > 1/2\).
\end{proof}
\begin{lemma}\label{lem: mgf q_t=-1 gaussian}
    With the same hypotheses as Lemma \ref{lem: mgf q_t=0, gaussian}
	\begin{align*}
		\E\left[e^{\lambda \Delta \|u_{t}\|_2^2}  \1_{q_t = -1} \1_U \1_{\|u_{t-1}\|_2 \geq \beta} \Big| \Fc_{t-1}\right] \leq \rho.
	\end{align*}
\end{lemma}
\begin{proof}
    The proof is effectively the same as that for Lemma \ref{lem: mgf q_t=1, gaussian}.
\end{proof}


\acks{This work was supported in part by National Science Foundation
Grant DMS-2012546 and a UCSD senate research award.}

\vskip 0.2in
\bibliography{citations}

\end{document}